\documentclass[10pt,twocolumn,letterpaper]{article}

\usepackage{cvpr}              %

\usepackage{float}
\newtheorem{theorem}{Theorem}

\newtheorem{lemma}[theorem]{Lemma}

\usepackage{url}
\usepackage{footnote}
\usepackage[bottom]{footmisc}
\makeatletter
\g@addto@macro{\UrlBreaks}{\UrlOrds}
\makeatother
\usepackage{makecell}

\usepackage{epsfig}
\usepackage{graphicx}
\usepackage{amsmath}
\usepackage{amssymb}
\usepackage{bbm}
\usepackage{soul}

\usepackage{float}
\usepackage{pifont}%

\usepackage{xspace}

\newcommand{\method}{$\phi$-DPO\xspace}

\usepackage{multirow}
\usepackage{color, colortbl}
\usepackage{booktabs}
\usepackage{algorithm, algpseudocode}
\usepackage{tcolorbox}
\usepackage{tabularx}
\usepackage{arydshln}
\usepackage{caption}

\usepackage{dblfloatfix}

\usepackage{tikz}
\usetikzlibrary{bayesnet}
\usetikzlibrary{arrows}
\usepackage{float}

\usepackage{wrapfig}

\definecolor{main}{HTML}{5989cf}    %
\definecolor{sub}{HTML}{cde4ff}     %
\definecolor{greyborder}{HTML}{808080}
\definecolor{lightgrey}{HTML}{D3D3D3}

\tcbset{
    sharp corners,
    colback = white,
    before skip = 0.2cm,    %
    after skip = 0.5cm      %
}                           %

\newtcolorbox{boxA}{
    fontupper = \bf,
    boxrule = 1.5pt,
    colframe = black %
}

\newtcolorbox{boxB}{
    fontupper = \bf\color{main}, %
    boxrule = 1.5pt,
    colframe = main,
    rounded corners,
    arc = 5pt   %
}

\newtcolorbox{boxC}{
    colback = sub, %
    boxrule = 0pt  %
}

\newtcolorbox{boxD}{
    colback = lightgrey, 
    colframe = greyborder, 
    boxrule = 0pt, 
    toprule = 3pt, %
    bottomrule = 3pt %
}

\newtcolorbox{boxE}{
    enhanced, %
    boxrule = 0pt, %
    borderline = {0.75pt}{0pt}{main}, %
    borderline = {0.75pt}{2pt}{sub} %
}

\newtcolorbox{boxF}{
    colback = sub,
    enhanced,
    boxrule = 1.5pt, 
    colframe = white, %
    borderline = {1.5pt}{0pt}{main, dashed} %
}

\newtcolorbox{boxG}{
    enhanced,
    boxrule = 0pt,
    colback = sub,
    borderline west = {1pt}{0pt}{main}, 
    borderline west = {0.75pt}{2pt}{main}, 
    borderline east = {1pt}{0pt}{main}, 
    borderline east = {0.75pt}{2pt}{main}
}

\newtcolorbox{boxH}{
    colback = sub, 
    colframe = main, 
    boxrule = 0pt, 
    leftrule = 6pt %
}

\newtcolorbox{boxI}{
    colback = sub, 
    colframe = main, 
    boxrule = 0pt, 
    toprule = 6pt %
}

\newtcolorbox{boxJ}{
    sharpish corners, %
    colback = sub, 
    colframe = main, 
    boxrule = 0pt, 
    toprule = 4.5pt, %
    enhanced,
    fuzzy shadow = {0pt}{-2pt}{-0.5pt}{0.5pt}{black!35} %
}

\newtcolorbox{boxK}{
    sharpish corners, %
    boxrule = 0pt,
    toprule = 4.5pt, %
    enhanced,
    fuzzy shadow = {0pt}{-2pt}{-0.5pt}{0.5pt}{black!35} %
}

\newtcolorbox{boxL}{
    fontupper = \color{main},
    rounded corners,
    arc = 6pt,
    colback = sub, 
    colframe = main!50, 
    boxrule = 0pt, 
    bottomrule = 4.5pt 
}

\newtcolorbox{boxM}{
    fontupper = \color{white},
    rounded corners,
    arc = 6pt,
    colback = main!80, 
    colframe = main, 
    boxrule = 0pt, 
    bottomrule = 4.5pt,
    enhanced,
    fuzzy shadow = {0pt}{-3pt}{-0.5pt}{0.5pt}{black!35}
}

\definecolor{cvprblue}{rgb}{0.21,0.49,0.74}
\usepackage[pagebackref,breaklinks,colorlinks,allcolors=cvprblue]{hyperref}

\def\paperID{} %
\def\confName{CVPR}
\def\confYear{2026}

\title{$\phi$-DPO: Fairness Direct Preference Optimization Approach to \\ Continual Learning in Large Multimodal Models
\vspace{-5mm}
}

\author{
Thanh-Dat Truong$^{1}$, Huu Thien Tran$^{1}$, Jackson Cothren$^{2}$, Bhiksha Raj$^{3}$, Khoa Luu$^{1}$\\
$^{1}$CVIU Lab, University of Arkansas, USA \quad
$^{2}$Dep. of  Geosciences, University of Arkansas, USA \\
$^{3}$Carnegie Mellon University, USA  \\
\tt\small \{tt032, ht035, jcothre, khoaluu\}@uark.edu, bhiksha@cs.cmu.edu\\
\small{\url{http://uark-cviu.github.io/projects/Fai-DPO}}
\vspace{-4mm}
}

\begin{document}

\maketitle

\begin{abstract}

Fairness in Continual Learning for Large Multimodal Models (LMMs) is an emerging yet underexplored challenge, particularly in the presence of imbalanced data distributions that can lead to biased model updates and suboptimal performance across tasks. While recent continual learning studies have made progress in addressing catastrophic forgetting, the problem of fairness caused by imbalanced data remains largely underexplored. This paper presents a novel Fairness Direct Preference Optimization (FaiDPO or $\phi$-DPO) framework for continual learning in LMMs. In particular, we first propose a new continual learning paradigm based on Direct Preference Optimization (DPO) to mitigate catastrophic forgetting by aligning learning with pairwise preference signals. Then, we identify the limitations of conventional DPO in imbalanced data and present a new $\phi$-DPO loss that explicitly addresses distributional biases. We provide a comprehensive theoretical analysis demonstrating that our approach addresses both forgetting and data imbalance. Additionally, to enable $\phi$-DPO-based continual learning, we construct pairwise preference annotations for existing benchmarks in the context of continual learning. Extensive experiments and ablation studies show the proposed $\phi$-DPO achieves State-of-the-Art performance across multiple benchmarks, outperforming prior continual learning methods of LMMs.
\end{abstract}

\section{Introduction}
\label{sec:intro}

Large multimodal models (LMMs) have shown their strong performance as general-purpose assistants in various visual learning tasks \cite{liu2024visual, liu2024improved, liu2024llavanext, bai2023qwen, li2022blip, li2023blip, deitke2025molmo, nguyen2024insect}.
The success of LMMs typically relies on supervised finetuning on carefully curated, large-scale multi-task datasets.
In practical deployment, LMMs often suffer from performance degradation when encountering novel knowledge, tasks, or shifts in data distribution. However, fully retraining large models to incorporate new knowledge and capabilities is computationally expensive and time-consuming. Meanwhile, direct fine-tuning on the new datasets may result in a performance drop for previously learned tasks \cite{zhao2025mllm}. This phenomenon is called the catastrophic forgetting problem.
In addition, although recent Retrieval-Augmented Generation (RAG) improves the contextual understanding of LMMs for more accurate responses \cite{jaiswal2025multimodal, yu2024visrag, huybrechts2025document, tanaka2025vdocrag, serna2026nico}, RAG-LMM systems struggle with distribution shifts and novel tasks. Since the retrieval only augments the input without updating model parameters, the internal knowledge representations remain unchanged.
Therefore, to ensure the reliable performance of the LMMs in a dynamic and adaptive environment, it is crucial to develop a \textbf{Continual Learning paradigm for LMMs} (Figure \ref{fig:highlight}) so that the models can incrementally acquire new information and skills while preserving the knowledge learned previously.

\begin{figure}
    \centering
    \includegraphics[width=1.0\linewidth]{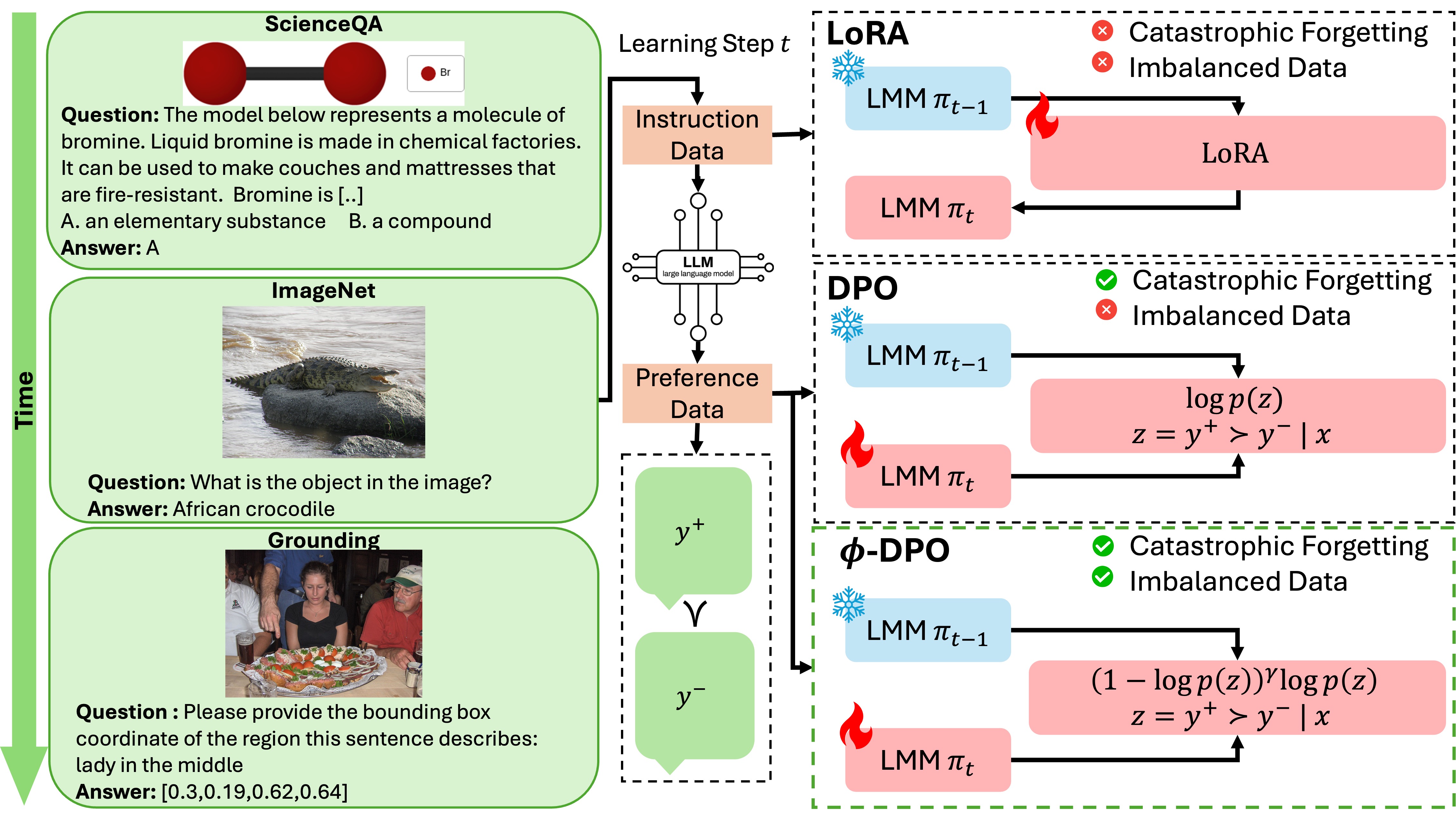}
    \vspace{-6mm}
    \caption{
    Our \textbf{Fairness DPO ($\phi$-DPO) approach to Continual Learning in LMMs}.
    Prior continual learning methods, \eg, \textbf{LoRA}, struggle under \textbf{imbalanced multimodal data} and suffer from \textbf{catastrophic forgetting}. The vanilla DPO is still influenced by the \textbf{imbalanced data distributions}. Our $\phi$-DPO approach can (1) mitigate forgetting, (2) adapt continuously to new learning tasks, and (3) maintain robustness under data imbalance. 
    }
    \label{fig:highlight}
    \vspace{-6mm}
\end{figure}

In continual learning, two primary challenges are identified: (1) Catastrophic Forgetting and (2) Fairness. 
Fairness in LMMs is particularly crucial for real-world deployment, since the biased or inconsistent behaviors can result in unequal outcomes and reduce trustworthiness, especially in human-centric applications.
Recent studies of continual learning in LMMs \cite{chen2024coin, zhao2025mllm} have introduced several methods to address the catastrophic forgetting problem. 
However, \textbf{fairness remains a fundamental issue in the continual learning of LMMs that has been unexplored}. 

\begin{figure}[!t]
    \centering
    \includegraphics[width=1.0\linewidth]{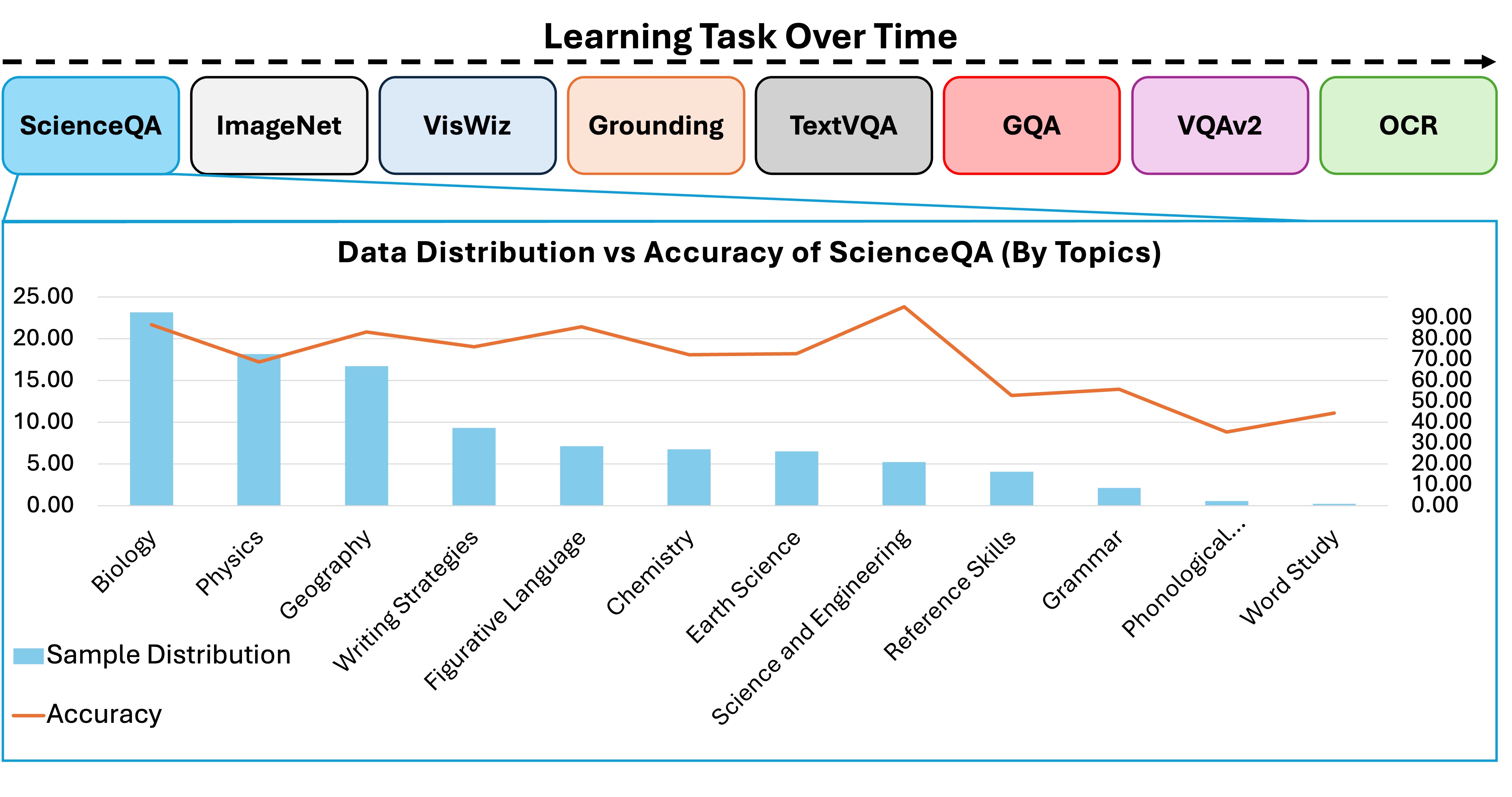}
    \vspace{-8mm}
    \caption{
    \textbf{The Imbalanced Distribution of Multimodal Continual Learning Benchmarks.} 
    The distribution of samples across ScienceQA topics is highly skewed, \ie categories with fewer training examples (\eg Grammar, Phonological Awareness, Word Study) exhibit significantly lower accuracy, while topics with richer data (\eg Biology, Physics) achieve stronger performance.
    }
    \label{fig:data-distribution}
    \vspace{-6mm}
\end{figure}

As shown in Figure~\ref{fig:data-distribution}, multimodal datasets often exhibit significant topic imbalance, introducing bias during each incremental task and resulting in skewed performance. This imbalance poses two key challenges: (1) Representation Alignment and (2) Adaptability and Forgetting. Unlike traditional LMMs trained on diverse modality pairs in a single stage, continual multimodal learning proceeds sequentially, requiring incremental alignment. Under imbalanced conditions, this leads to biased gradient updates across tasks, groups, or domains (see Section~\ref{sec:method}).
For example, Figure~\ref{fig:data-shift} illustrates the progression of training across tasks, starting with ScienceQA, which focuses on structured visual reasoning through diagram-text alignment. The subsequent Grounding task redirects attention toward object-level localization, whereas OCR-VQA centers on fine-grained text extraction. These tasks exhibit distinct visual distributions, language prompts, and alignment objectives, ultimately leading to modality imbalance and a degradation of prior representation alignment.
As each task may be dominated by a particular modality or semantic class, gradient updates often favor current majority signals, undermining prior alignment and degrading performance on earlier tasks. These imbalances undermine representation alignment and may exacerbate catastrophic forgetting.
Moreover, prior continual learning studies have focused on unimodal settings \cite{cermelli2020modelingthebackground, douillard2021plop, sats_prj_2023, truong2025falcon, truong2023fairness}, whereas LMMs introduce new complexity due to their multimodal settings.
In this context, biased data becomes a critical issue, as the imbalance in multi-modalities or semantic classes not only exaggerates catastrophic forgetting but also limits the adaptability of LMMs to new tasks or knowledge.

\begin{figure}[!t]
    \centering
    \includegraphics[width=1.0\linewidth]{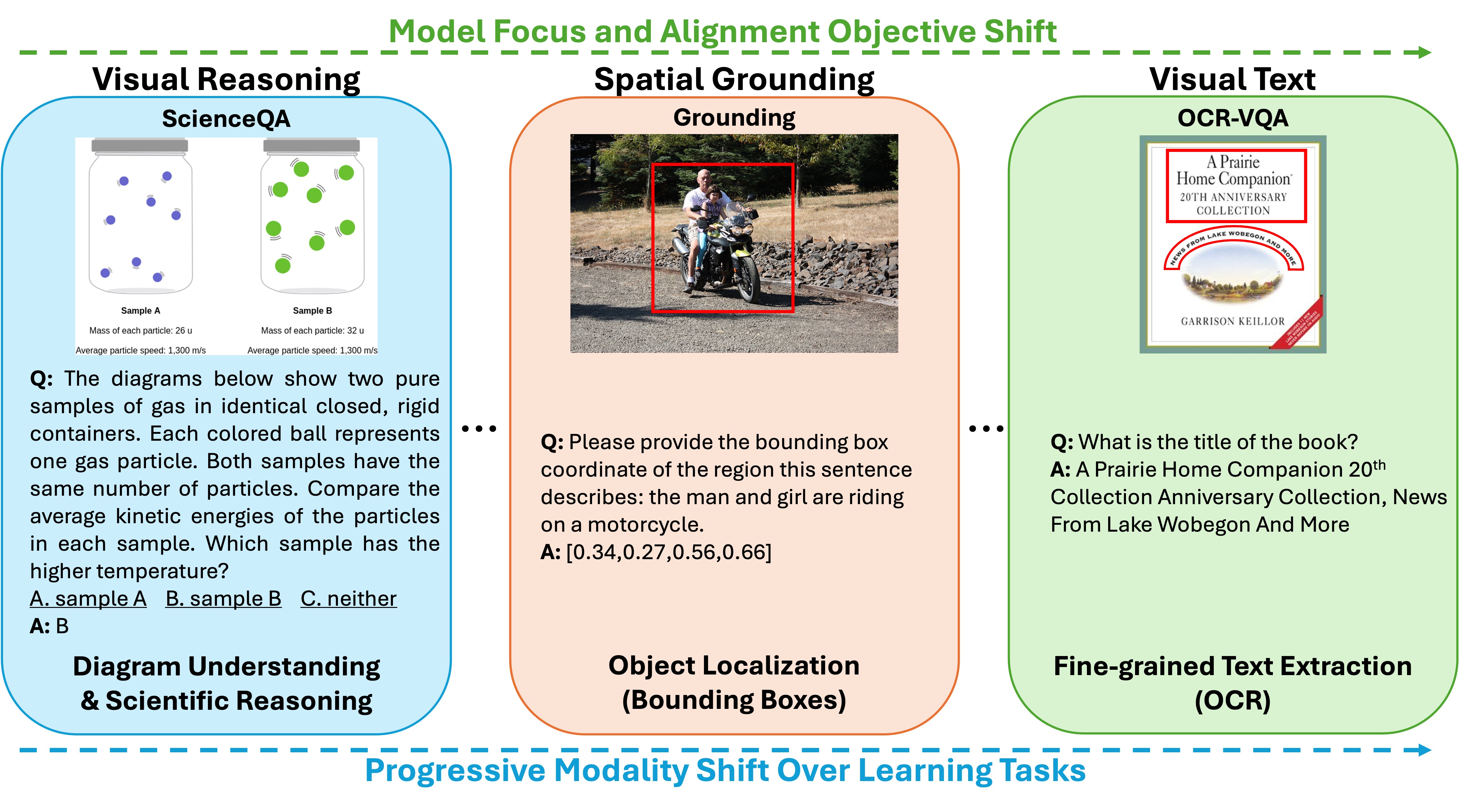}
    \vspace{-8mm}
    \caption{ScienceQA, Grounding, and OCR-VQA introduce progressively shifting visual distributions and alignment objectives, creating modality imbalance across tasks.}
    \label{fig:data-shift}
    \vspace{-6mm}
\end{figure}

While data imbalance poses critical challenges, current methods remain inadequate in addressing this problem. 
Low-rank Adaptation (LoRA) \cite{hu2022lora} is widely used in continual learning for LMMs \cite{chen2024coin, guo2025hidellava, zhao2025mllm}. Although it preserves the frozen backbone, its adapters can inherit dataset bias, leading to gradient updates skewed toward majority semantic classes \cite{chang2025balora, wang2024exploring}. Moreover, LoRA does not inherently mitigate bias propagation \cite{chang2025balora} and is prone to catastrophic forgetting, especially when adapters are shared or new ones induce representation drift \cite{zhang2025c, han2024slim}.
Meanwhile, Knowledge Distillation has also been widely adopted in unimodal continual learning \cite{cermelli2020modelingthebackground, douillard2021plop, sats_prj_2023}. However, it remains limited in multimodal contexts. LMMs often encode demographic and distributional biases from their large-scale pretraining, which distillation can transfer or amplify in the student model \cite{zhang2023not, shi2024densely}. Imitating biased teacher outputs could lead to suboptimal and biased predictions \cite{zhang2023not}. Under imbalance, majority data dominate distillation gradients, lowering generalization to tail classes \cite{truong2023fairness, truong2025falcon}. In addition, while distillation aligns output probabilities, it fails to preserve internal representations critical for maintaining prior knowledge, particularly in LMMs where knowledge spans multiple modalities and layers \cite{aguilar2020knowledge, tian2019contrastive}.

\noindent
\textbf{Contributions.}
This work proposes a novel Fairness Direct Preference Optimization (FaiDPO or $\phi$-DPO) approach to Continual Learning in LMMs.
Our contributions can be summarized as follows.
First, we introduce a new continual learning paradigm, Direct Preference Optimization (DPO), which addresses the catastrophic forgetting problem in continual learning.
Second, by analyzing the limitations of traditional DPO, we present a new Fairness DPO loss to address the fairness problem caused by imbalanced data.
We provide a comprehensive theoretical analysis to show that our proposed approach can address both catastrophic forgetting and imbalanced data.
Third, to support the DPO learning in our framework, we contribute the DPO labels for current continual learning benchmarks.
Finally, our intensive experiments and ablation studies have illustrated the effectiveness and State-of-the-Art performance (SoTA) of our approach compared to prior continual learning methods.

\section{Related Work}
\label{sec:related}
{
}
\noindent
\textbf{Large Multimodal Models.}
Early advances in large language models (LLMs) \cite{li2025llama, vicuna2023, achiam2023gpt, siino2024mcrock, bai2023qwen} have driven rapid progress in LMMs. 
Recent models span across vision-language \cite{liu2024visual, liu2024improved, truong2025directedtokens, truong2025insect, tran2025bima, truong2025mango}, video-language \cite{weng2024longvlm, zhao2023learning, lin2023video, li2023videochat, nguyen2025hyperglm}, and audio-language domains \cite{hussain2023m, ghosal2023text}, with LMMs playing a central role in scaling multimodal understanding.
{%
}
Early work by \cite{alayrac2022flamingo} effectively bridged vision and language modalities in a few-shot learning setting, followed by \cite{li2023blip}, which enhanced inter-modality connectivity via Q-Former. 
This was further developed into an instruction-aware model within the vision-language instruction-tuning framework \cite{liu2024visual}.
LLaVA \cite{liu2024visual} established a streamlined visual-to-language space projection using a linear layer, later refined by \cite{liu2024improved} with an MLP and AnyRes, a technique adept at handling high-resolution images. Subsequent studies \cite{liu2024llavanext, li2024llavanext-strong, zhang2024llavanextvideo, li2024llavanext-ablations, li2024llavanext-interleave} contributed further improvements, culminating in a robust model \cite{li2024llava1vision} capable of handling diverse vision tasks. 
These multimodal models are finding real-world use across a variety of domains. For instance, they have been adapted for biomedical analysis \cite{li2024llava}, improved through multimodal federated learning \cite{chenfedmbridge}, and applied to 3D point cloud understanding \cite{yang2023lidar, xu2024pointllm, Liu2024Uni3DLLMUP}.
Recent work \cite{bai2023qwenvl, wang2025internvl35, shen2025skyworkr1v3} with more effective training techniques and better architectures, serve as a stepping stone for a wide range of more advanced research efforts aimed at developing generalist LMMs.

\noindent
\textbf{Continual Learning.} 
The topic of continual learning has evolved through several core paradigms, each approaching the stability-plasticity dilemma from a unique angle. 
The field has largely centered around rehearsal-base approaches~\cite{lavda2018continual, buzzega2020dark}, regularization-based methods~\cite{truong2023fredom, truong2024eagle, truong2021bimal, kirkpatrick2017overcoming, li2017learning, nguyen2026cobra}, structure-based strategies~\cite{truong2023fairness, truong2025falcon, mallya2018piggyback, douillard2022dytox}, and prompt-based methods~\cite{wang2022learning, smith2023coda}.
In parallel, continual learning for large language models has rapidly become a focal point of recent work~\cite{shi2024continual_survey}.
Depending on where adaptation occurs, current efforts can be categorized into three major stages:
continual pre-training~\cite{jang2022temporalwiki, cossu2024continual}, continual instruction tuning~\cite{razdaibiedina2023progressive, zan2022cert, yin2022contintin, wang2023orthogonal}, and continual alignment~\cite{zhang2024cppo, suhr2023continual}.
Meanwhile, for LMMs, progress remains relatively limited~\cite{chen2024coin, zeng2024modalprompt, cao2024continualllava, guo2025comprehensive, guo2025hidellava, zhao2025mllm}.
Chen \etal \cite{chen2024coin} introduced one of the first systematic benchmarks for continual instruction tuning of LMMs. 
Building on this, Zeng \etal \cite{zeng2024modalprompt} proposed a dual-modality guided prompt framework to improve efficiency and stability. 
Cao \etal \cite{cao2024continualllava} and Guo \etal \cite{guo2025hidellava} explored hierarchical and modular strategies to better preserve multimodal representations over sequential updates.
Chen \etal \cite{chen2025sefe} and Zhao \etal \cite{zhao2025mllm} further attempted to mitigate forgetting and enhance continual adaptability, while Lin \etal \cite{lin2025continual} proposed a novel paradigm of sparse memory fine-tuning.
While prior continual learning studies in unimodal learning has taken fairness into consideration \cite{truong2023fairness, truong2025falcon}, there are limited studies addressing this problem in LMMs.
Unlike prior methods, \textit{we propose to address the catastrophic forgetting problem via a new continual learning paradigm and achieve fairness} under imbalanced data settings.

\section{The Proposed $\phi$-DPO Approach}
\label{sec:method}

Given a large multimodal model $\pi$, continual learning involves incrementally learning it on a sequence of datasets $\mathcal{D}\!=\!\{\mathcal{D}_1, ..., \mathcal{D}_T\}$, where $T$ is the number of learning steps.
For each learning step $t$, $\mathcal{D}_t\!=\!\{x^j, y^j\}_{j=1}^{|\mathcal{D}_i|}$ constraints $|\mathcal{D}_i|$ instruction data, where $x^j\!=\!(x^j_\mathrm{img}, x^j_\mathrm{ins})$ consists of a pair of an image 
$x^j_\mathrm{img}$ and a textual instruction $x^j_\mathrm{ins}$, and $y^j$ is the corresponding answer. 
Formally, learning the LMM $\pi_t$ at step $t$ on dataset $\mathcal{D}_t$ can be formed as follows:
\begin{equation}\label{eqn:general-cl-mllm}
\footnotesize
    \pi^*_t = \arg\!\max_{\pi_t} \mathbb{E}_{x, y \in \mathcal{D}_t} \log p(y | x) + D_\mathrm{Forget}(\pi_t \| \pi_{t-1}) 
\end{equation}
where $\log p(y | x)$ is the supervised fine-tuning loss on the instruction data,
$D_\mathrm{Forget}(\pi_t \| \pi_{t-1})$ is the forgetting mitigation that prevents the current LMM $\pi_t$ drifted away from the previous learned LMM $\pi_{t-1}$, \ie, avoid forgetting.

Prior continual learning studies commonly adopt knowledge distillation \cite{cermelli2020modelingthebackground, douillard2021plop} to mitigate the forgetting. However, as shown in Sec.~\ref{sec:intro}, the traditional knowledge distillation may exaggerate the bias and cause catastrophic forgetting, especially in the context of multimodal learning.
Several recent studies adopt contrastive clustering \cite{truong2025falcon, truong2023fairness} to model catastrophic forgetting. Although it has yielded promising results in unimodal problems, defining clusters in multimodal settings is infeasible. Thus, to address this problem, our approach adopts Reinforcement Learning from Human Feedback (RLHF) to model the forgetting.

Formally, let $r(x, y)$ be the reward model to evaluate the forgetting and adaptability level of $\pi_t$, \ie, the higher $r(x, y)$, the better memory retained and adaptability. Then, to model catastrophic forgetting ($D_\mathrm{Forget}$), our continual learning can be reformulated under an RLHF perspective as in Eqn. \eqref{eq:rlhf_constrained}.
\begin{equation}
\footnotesize
\label{eq:rlhf_constrained}
\begin{split}
\pi_t^\star \;=\; \arg\!\max_{\pi_t}\; \mathbb{E}_{x\sim \mathcal{X}_t}\,\mathbb{E}_{y\sim \pi_t(\cdot\mid x)}\big[ r(x,y)\big] \\
\quad\text{s.t.}\quad
D_{\mathrm{KL}}\!\big(\pi_t(\cdot\mid x)\,\big\|\,\pi_{t-1}(\cdot\mid x)\big)\le \delta,
\end{split}
\end{equation}
where $\pi_{t-1}$ is the reference (previous learning step) policy, $D_\textrm{KL}(\pi_t \| \pi_{t-1})$ is the KL divergence to measure the difference of predictions between previous and current LMMs, $\delta$ is the threshold to constraint the policy update. 

Although learning Eqn.~\eqref{eq:rlhf_constrained} can be achieved via Proximal Policy Optimization Algorithms (PPO), it still presents two major challenges.
First, learning the reward model $r$ requires the corresponding training data. In addition, in the context of continual learning, it may require incrementally learn the reward model $r$ at each learning step, which is not feasible.
Second, learning the reward model via PPO on the imbalanced data may lead to biased predictions produced by the model \cite{schulman2017proximal}.
To address these challenges, inspired by \cite{rafailov2023direct}, we propose modeling RLHF reward in continual learning via Direct Preference Optimization. We introduce a new Fairness DPO loss to address fairness modeling.

\subsection{Direct Preference Optimization to Continual Learning in LMMs}

\subsubsection{DPO as Continual Learning}

Inspired by \cite{rafailov2023direct}, learning the LMM at learning step $t$ of Eqn. \eqref{eq:rlhf_constrained} via RLHF can be rewritten using the Lagrangian multiplier as follows:
\begin{equation}\label{eq:rlhf-beta}
\footnotesize
\begin{split}
    \pi_t^\star = \max_{\pi_t}\; \mathbb{E}_{x, y\sim \pi_t}[r(x,y)] - \beta D_{\mathrm{KL}}\!\big(\pi_t(\cdot\mid x)\,\big\|\,\pi_{t-1}(\cdot\mid x)\big)
\end{split}
\end{equation}
where $\beta$ is the Lagrangian multiplier that controls the divergence of the LMM at the current learning step from the previous step.
As shown in \cite{rafailov2023direct}, the learning objective in Eqn.~\eqref{eq:rlhf-beta} achieves optimal when it takes the following forms:
\begin{equation}
\footnotesize
\label{eq:boltzmann}
\begin{split}
\pi_t^\star(y\mid x) &= \frac{\pi_{t-1}(y\mid x)\,\exp\left(\tfrac{1}{\beta}\,r(x,y)\right)}{Z(x)},
\\
r(x, y) &= \beta\log\frac{\pi_t^\star(y\mid x)}{\pi_{t-1}(y\mid x)} +  \beta \log Z(x),
\end{split}
\end{equation}
where $Z(x) = \sum_{y}\pi_{t-1}(y|x)\exp\left(\frac{1}{\beta}r(x, y)\right)$.

\noindent
\textbf{From RLHF to Pairwise Preferences.}
In our continual learning setting, for each sample instruction data $x$, let $y^+$ be the well-retrained (good memory) and well-adapted output, and let $y^-$ be the forgotten output. The preference of $y^+$ over $y^-$, \ie, $p(y^+\succ y^- | x)$ can be formed via the  Bradley-Terry model as follows:
\begin{equation}
\footnotesize
\begin{split}
    p(y^+\succ y^- | x) = \sigma\!\big(\beta\,[r(x,y^+)-r(x,y^-)]\big),
\end{split}
\end{equation}
where $\sigma(u) = \frac{1}{1+\exp(-u)}$ is the Sigmoid function. 
Then, maximizing the (conditional) log-likelihood over pairs to avoid catastrophic forgetting in continual learning can be reformed as the following logistic loss:
\begin{equation}
\footnotesize
\label{eq:loss-pref}
\mathcal{L}_{\text{pref}}(\pi_t, \pi_{t-1})
=
\mathbb{E}_{(x,y^+,y^-)}
\Big[
-\log \sigma \big(\beta \big[r(x,y^+)-r(x,y^-)\big]\big)
\Big]
\end{equation}

\noindent
\textbf{Continual Learning Paradigm with DPO.}
Under the optimum conditions defined in Eqn.~\eqref{eq:boltzmann}, 
the logistic loss of the \emph{reward difference} in Eqn.~\eqref{eq:loss-pref} can be written in terms of policy log-ratios as in Eqn. \eqref{eqn:dpo_loss}.
\begin{equation}\label{eqn:dpo_loss}
\footnotesize
\begin{split}
    &\mathcal{L}_{\mathrm{DPO}}(\pi_t, \pi_{t-1}) = 
    -\mathbb{E}_{x, y^+, y^-}\Bigg[
    \log\sigma\Bigg(\beta\log\frac{\pi_t(y^+|x)}{\pi_{t-1}(y^+|x)}  
    \\
    &\quad\quad\quad\quad\quad\quad\quad\quad\quad\quad\quad\quad\quad\quad\quad\quad - \beta \log\frac{\pi_t(y^-|x)}{\pi_{t-1}(y^-|x)}\Bigg)
    \Bigg] \\
    &= 
    -\mathbb{E}_{x, y^+, y^-} 
    \log \sigma \Bigg[
    \beta \underbrace{\Big[\Big(\log \pi_t(y^+| x)-\log \pi_t(y^-| x)\Big]}_{\text{trainable policy}}
    \\ &\quad\quad\quad\quad\quad\quad\quad\quad\quad  - \beta\underbrace{\Big[\log \pi_{t-1}(y^+| x)-\log \pi_{t-1}(y^-| x)\Big]}_{\text{fixed reference}}
    \Bigg]
    \raisetag{17mm}
\end{split}
\end{equation}

\noindent
\textbf{Interpretation.}
As in continual learning, each training stage aims to adapt the current policy $\pi_{t}$ to a new learning task while preserving consistency with the reference model $\pi_{t-1}$ to avoid catastrophic forgetting. 
The DPO objective encourages $\pi_t$ to increase the relative log-odds of well-retained (or memory-consistent) $y^+$ over forgotten output $y^-$ compared to $\pi_{t-1}$, effectively favoring outputs that remain aligned with prior knowledge. 
Our learning objective prevents the policy from drifting away from the previous learning task and implicitly regularizes updates toward the information manifold of the previous model, thereby mitigating catastrophic forgetting.
Meanwhile, the hyper-parameter $\beta$ controls the adaptability of the LMM, \ie, larger values constrain $\pi_t$ to remain close to $\pi_{t-1}$ (\textit{stability}). In comparison, smaller values permit more flexible adaptation to new distributions (\textit{plasticity}). 
Our learning mechanism replaces the explicit reward-based RLHF objective in Eqn.~\eqref{eq:rlhf_constrained} with a preference-based contrastive loss that directly enforces policy consistency under a bounded divergence constraint. 
As a result, our DPO approach provides a principled mechanism for continual alignment, \ie, preserving prior knowledge while enabling controlled updates that enhance adaptability across sequential tasks.

\subsubsection{Theoretical Analysis of Direct Preference Optimization in Forgetting Mitigation}

Prior studies \cite{douillard2021plop, cermelli2020modelingthebackground} have shown that Knowledge Distillation (KD) is a common approach to mitigate catastrophic forgetting. In this section, we will provide a theoretical analysis to demonstrate that the knowledge distillation loss is bounded by the DPO loss, which offers a more effective mechanism for preventing catastrophic forgetting and enhancing adaptability.
In a typical knowledge distillation approach used in continual learning, 
the current model $\pi_t$ is encouraged to remain close to the previous model $\pi_{t-1}$ via the Kullback–Leibler (KL) divergence to mitigate catastrophic forgetting: 
\begin{equation}
\footnotesize
\begin{split}
\mathcal{L}_{\mathrm{KD}}(\pi_t, \pi_{t-1}) &= D_{\mathrm{KL}}(\pi_{t-1} \,\|\, \pi_t) = \mathbb{E}_{x, y \sim \pi_{t-1}} \left[ \log\frac{\pi_{t-1}(y | x)}{\pi_t(y | x)} \right]
\end{split}
\end{equation}

\begin{lemma}\label{lemma:lower-bound-kl}
\textbf{Lower Bound of KL Divergence Governed by DPO Loss.} 
The lower bound of the $D_{\mathrm{KL}}(\pi_{t-1}\|\pi_t)$ is governed by the DPO loss as follows:
\begin{equation}\label{eqn:lowerbound}
\footnotesize
        D_{\mathrm{KL}}(\pi_{t-1}\|\pi_t) \geq \frac{1}{C_{\mathrm{lower}}} (\log 2 - \mathcal{L}_{\mathrm{DPO}}(\pi_t;\pi_{t-1}))^2
\end{equation}
where $C_{\mathrm{lower}}$ is a constant number.
\end{lemma}

\begin{lemma}\label{lemma:upper-bound-kl}
\textbf{Upper Bound of KL Divergence Governed by DPO Loss.} 
The upper bound of the $D_{\mathrm{KL}}(\pi_{t-1}\|\pi_t)$ is governed by the DPO loss as follows:
\begin{equation}\label{eqn:upperbound}
\footnotesize
        D_{\mathrm{KL}}(\pi_{t-1}\|\pi_t) \le C_{\mathrm{upper}} \mathcal{L}_{\mathrm{DPO}}(\pi_t;\pi_{t-1})
\end{equation}
where $C_{\mathrm{upper}}$ is a constant number.
\end{lemma}

\noindent
\textbf{Proof.} The proof of Lemmas \ref{lemma:lower-bound-kl}-\ref{lemma:upper-bound-kl} is in our appendix. We show the exact form of $C_{\mathrm{lower}}$ and $C_{\mathrm{upper}}$ in our proof.

\noindent
\textbf{Interpretation.}
Lemmas~\ref{lemma:lower-bound-kl} and~\ref{lemma:upper-bound-kl} establish a two-sided relationship between the DPO loss and the KL divergence used in prior knowledge distillation methods \cite{cermelli2020modelingthebackground, douillard2021plop}. The lower bound implies that a small DPO loss ensures $\pi_t$ remains close to $\pi_{t-1}$, thereby preserving prior knowledge and mitigating catastrophic forgetting.
The upper bound further constrains the KL divergence, indicating that DPO regularizes updates so that $D_{\mathrm{KL}}(\pi_{t-1}\|\pi_t)$ grows proportionally to the DPO loss.
These bounds indicate that DPO implicitly controls catastrophic forgetting and adaptation, \ie a small $\mathcal{L}_{\mathrm{DPO}}$ constrains the divergence, ensuring semantic consistency with $\pi_{t-1}$ while allowing new learning task updates. 
Different from KD, which minimizes KL divergence directly, DPO introduces an adaptive, pairwise preference mechanism.
Therefore, DPO can be viewed as a generalized form of distillation, \ie, retaining the regularization effect while selectively amplifying high-reward (well-retained) responses and suppressing low-reward (forgotten) ones. This makes DPO more robust to forgetting while maintaining flexibility for continual learning in LMMs.

\subsection{Fairness DPO in Continual Learning}

\begin{figure}
    \centering
    \includegraphics[width=1.0\linewidth]{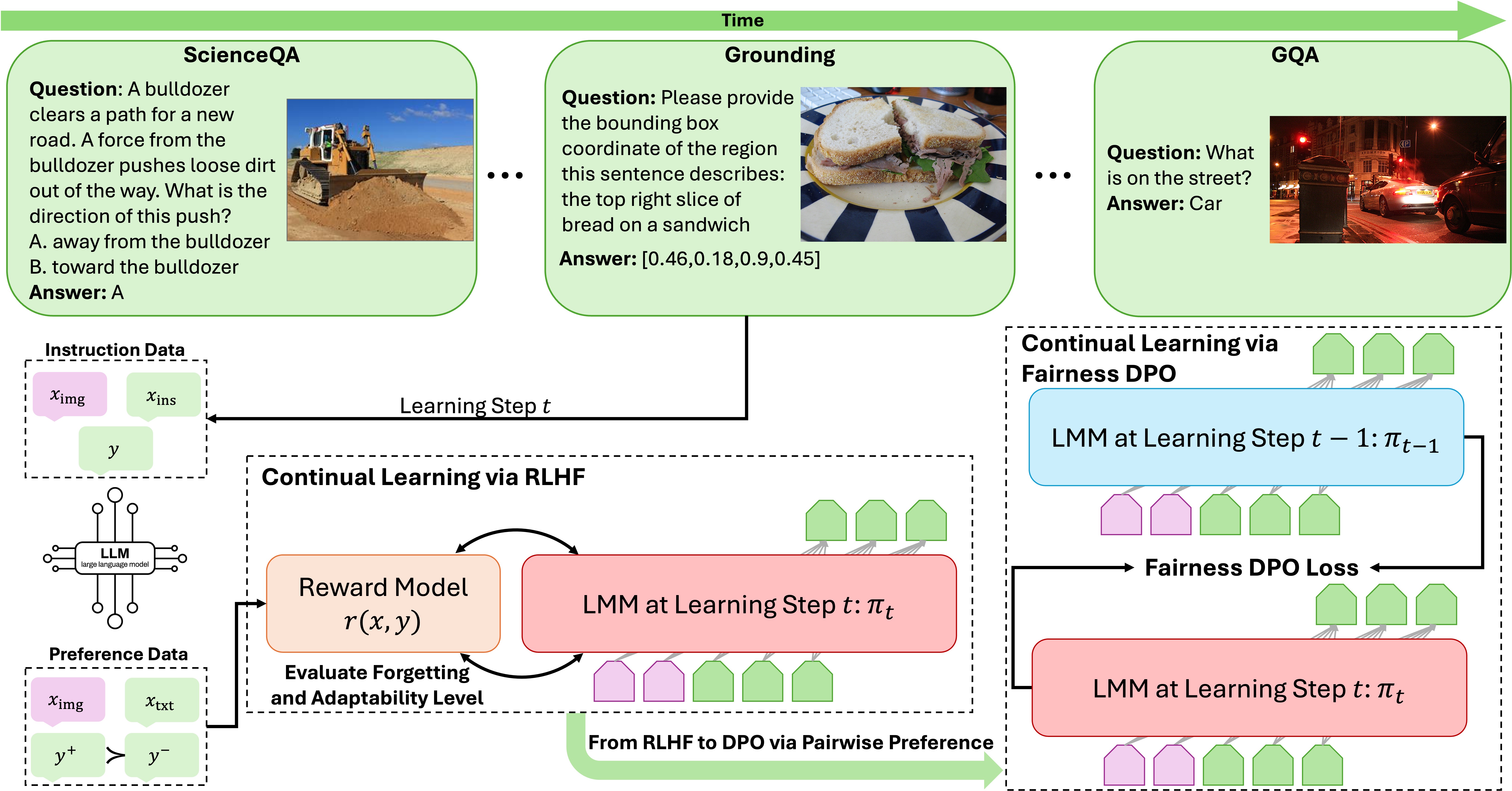}
    \vspace{-6mm}
    \caption{
    \textbf{Our Proposed Continual Learning Approach via Fairness DPO for Large Multimodal Models}. Traditional reinforcement learning with human feedback (RLHF) method optimize models through explicit reward maximization. Our framework instead reformulates RLHF as Direct Preference Optimization (DPO). The Fairness DPO loss mitigate the gradient biased under the imbalanced data.
    }
    \label{fig:framework}
    \vspace{-6mm}
\end{figure}

While the vanilla DPO loss can help prevent catastrophic forgetting and improve adaptability, the imbalance in data distribution can influence the behavior of DPO, leading to suboptimal performance. Indeed, let us revise the gradient produced by the DPO loss defined in Eqn. \eqref{eqn:dpo_loss} as follows:
\begin{equation}
\footnotesize
\begin{split}
\nabla_\theta \mathcal{L}_{\mathrm{DPO}}(\theta;\mu)
& = \mathbb E_\mu\!\big[(p(z)-1)\,\nabla_\theta s_\theta(z)\big] \\
    \mathcal{L}_{\mathrm{DPO}}(\theta;\mu) & = -\mathbb E_{(x,y^+,y^-)\sim \mu}\big[\log p(y^+\succ y^- | x)\big],
\end{split}
\end{equation}
where $\theta$ is the parameters of LMM, $z = (x, y^+, y^-)$, and $s_\theta(z) = \beta \big(r(x,y^+)-r(x,y^-)\big)$, $\mu$ is the data distribution of the current learning task.
Let us define $p(z) = p(y^+\succ y^- | x)$. In addition, we partition training data into $K$ disjoint groups $\{G_k\}_{k=1}^K$, with mixture weights $\mu_k=\mu(G_k)$.
Then, the group gradients can be rewritten as follows:
\begin{equation}\label{eqn:dpo-gradient}
\footnotesize
\begin{split}
    \nabla_\theta \mathcal{L}_{\mathrm{DPO}}(\theta;\mu) &= \sum_{k=1}^K \mu_k\, m_k(\theta) \\
    m_k(\theta) &= \mathbb E\!\big[(p(z)-1)\nabla_\theta s_\theta(z)\mid z\in G_k\big]
\end{split}
\end{equation}
where $m_k(\theta)$ is the group mean gradient.
Let $q'=(q'_1,\dots,q'_K)$ be the desired (ideal) balanced mixture over groups, and
$q=(q_1,\dots,q_K)$ the observed imbalanced distribution, with $p\neq q$. 
In an ideal scenario, the LLM $\pi_{t}$ trained on the balanced data distribution $q'$ will perform fairly.
Then, the gradient difference incurred when optimizing with biased distribution $q$ instead of balanced distribution $q'$ can be defined as:
\begin{equation}\label{eqn:grad-diff}
\footnotesize
B(\theta)
= \nabla_\theta \mathcal{L}_{\mathrm{DPO}}(\theta;q)-\nabla_\theta \mathcal{L}_{\mathrm{DPO}}(\theta;q')
= \sum_{k=1}^K (q_k-q'_k)\,m_k(\theta)
\end{equation}
Then, suppose there exists a group $j$ such that $q_j>q'_j$ (major group), the $j$-th group contributes $(q_j-q'_j)\, m_j(\theta)$ in the gradient updates, which is a systematic overweighting of group $j$’s gradient. Meanwhile, if $q_i<q'_i$ for a minority group $i$, then the $i$-th term is underweighted.
In other words, the gradient updates produced by the vanilla DPO loss will be biased towards major groups, and the gradient difference between ideal and practical data distribution will not be identical, \ie, $\|B(\theta)\| \neq 0$.

\begin{figure*}[!b]
    \centering
    \vspace{-4mm}
    \includegraphics[width=0.8\linewidth]{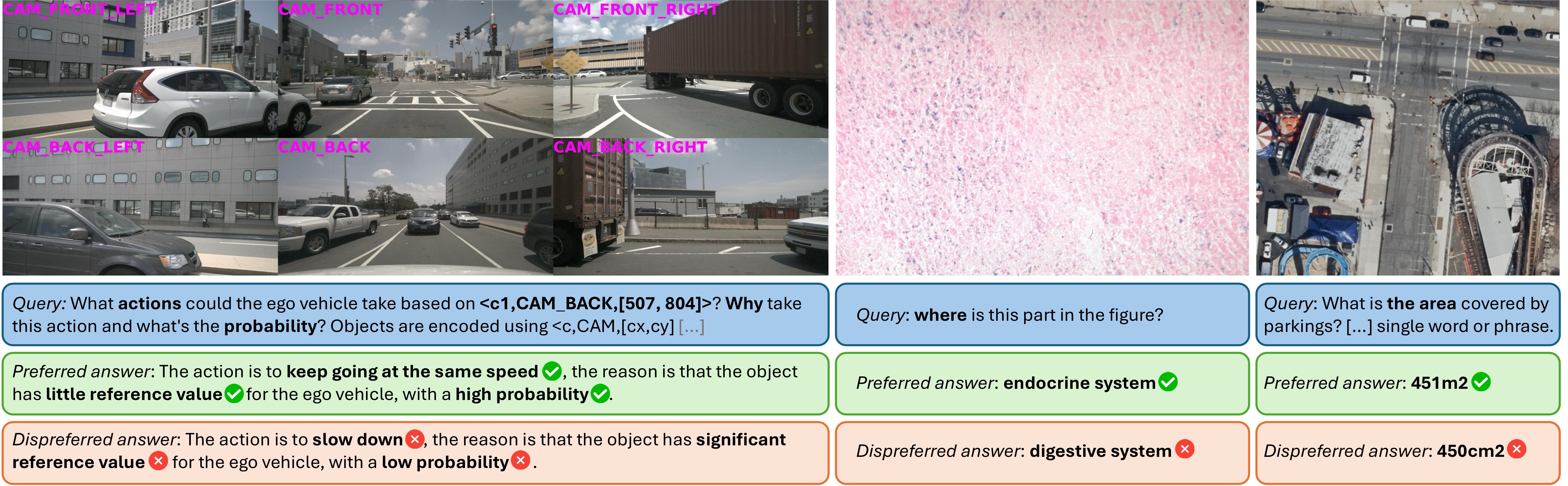}
    \vspace{-3mm}
    \caption{Example of Our DPO Data in the Continual Learning Benchmark. Best viewed in color.}
    \label{fig:example-dpo}
\end{figure*}

To address the problem caused by the imbalanced data distribution, inspired by the focal loss \cite{lin2017focal}, we introduce a new Fair DPO loss to improve the fairness of the LMM.  In particular, the \textbf{Fair DPO loss} can be defined as follows:
\begin{equation}
\footnotesize
    \mathcal{L}^{\gamma}_{\mathrm{DPO}}(\theta;\mu) = -\,\mathbb E_{z\sim\mu}
\Big[(1-p(z))^\gamma \,\log p(z)\Big]
\end{equation}
where $\gamma$ is the focusing parameter. Then, the gradient updates in Eqn.~\eqref{eqn:dpo-gradient} can be rewritten as follows:
\begin{equation}\label{eqn:fair-dpo-gradient}
\footnotesize
\begin{split}
  \nabla \mathcal{L}^{\gamma}_{\mathrm{DPO}}(\theta;\mu) &= \sum_{k=1}^K \mu_k\,  w_k^\gamma(\theta)\, m_k(\theta) \\
  \text{where} \quad w_k^\gamma(\theta) &= \mathbb E\!\big[\alpha_\gamma(p(z))\mid z\in G_k\big] \\
  \text{and} \quad \alpha_\gamma(p) &= (1-p)^{\gamma-1}\Big[(1-p)+\gamma p\log p\Big]
\end{split}
\end{equation}
In Eqn.~\eqref{eqn:fair-dpo-gradient}, $w_k^\gamma(\theta)$ plays a role as a modulating factor to balance the gradients of each group. Then, the gradient difference in Eqn. \eqref{eqn:grad-diff} with respect to the Fair DPO loss can be rewritten as:
\begin{equation}
\footnotesize
\begin{split}
B_\gamma(\theta) &=\nabla \mathcal{L}^{\gamma}_{\mathrm{DPO}}(\theta;q)-\nabla \mathcal{L}^{\gamma}_{\mathrm{DPO}}(\theta;q')
\\
&=\sum_{k=1}^K (q_k-q'_k)\, w_k^\gamma(\theta)\, m_k(\theta)
\end{split}
\end{equation}

\begin{lemma}\label{lema:fair-dpo-loss}
\textbf{Balanced Gradient Update of Fair DPO Loss}. Given a sufficient large value of $\gamma$, the Fair DPO loss will produce the balanced gradient update across groups regardless of the biased data distribution, \ie $\lim_{\gamma\to\infty} \|B_\gamma(\theta)\|=0$.
\end{lemma}

\noindent
\textbf{Proof.} The proof of Lemma 3 is included in our appendix.

\noindent
\textbf{Interpretation.}
When the focusing parameter $\gamma$ becomes sufficiently large, the discrepancy between optimization over the imbalanced distribution and the idealized balanced distribution vanishes, \ie, $\lim_{\gamma \to \infty} B_\gamma(\theta) = 0$.
This indicates that our proposed Fair DPO loss can yield fairer gradient updates as $\gamma$ increases. 
However, excessively large values of $\gamma$ can cause gradient vanishing, leading to a numerically flat loss landscape. As a result, this limits the adaptability of the LMMs in continual learning settings. Meanwhile, if $\gamma$ is too small, the loss behaves similarly to the standard DPO objective, potentially reinforcing the unfairness induced by imbalanced data. Therefore, careful tuning of $\gamma$ is crucial to strike a balance between fairness and plasticity in the continual learning of LMMs.

\noindent
\textbf{Continual Learning Procedure.} Figure \ref{fig:framework} illustrates our continual learning framework. In particular, the final learning objective of our proposed approach at each learning step $t$ can be formed as follows:
\begin{equation}\label{eqn:final-cl-mllm}
\footnotesize
    \pi^*_t = \arg\!\min_{\pi_t} \mathbb{E}_{x, y \in \mathcal{D}_t} -\log p(y | x) + \mathcal{L}^{\gamma}_{\mathrm{DPO}}(\pi_t \| \pi_{t-1}) 
\end{equation}
To avoid the overfitting on small-scale size data and increase the memory-efficiency during training DPO, the LMM $\pi_t$ at learning step $t$ is optimized via LoRA.

\subsection{DPO Data in Continual Learning Benchmark}

We conduct our experiments on three benchmark suites: CoIN~\cite{chen2024coin}, MLLM-CL Domain \cite{zhao2025mllm}, and MLLM-CL Ability \cite{zhao2025mllm}.
The CoIN benchmark comprises eight diverse learning tasks: ScienceQA \cite{lu2022learn}, TextVQA \cite{singh2019towards}, ImageNet \cite{deng2009imagenet}, GQA \cite{hudson2019gqa}, VizWiz \cite{gurari2018vizwiz}, Grounding \cite{mao2016generation, kazemzadeh2014referitgame}, VQAv2 \cite{goyal2017making}, and OCR-VQA \cite{mishra2019ocr}.
The MLLM-CL Domain benchmark is designed for domain-incremental learning and consists of five sequential domains: Remote Sensing \cite{lobry2020rsvqa}, Medical \cite{he2020pathvqa}, Autonomous Driving \cite{sima2024drivelm}, Science \cite{kembhavi2016diagram, guo2025sciverse, chang2022mapqa, kembhavi2017you}, and Finance \cite{zhao2025mllm}.
The MLLM-CL Ability benchmark focuses on task-incremental learning and includes four tasks: OCR \cite{li2024monkey, liu2024ocrbench}, Math and Logic \cite{shi2024math, zhang2024mavis}, Visual Perception \cite{johnson2017clevr, acharya2019tallyqa}, and GUI Agent \cite{ying2024mmt, hsiao2022screenqa, wang2021screen2words}.

While these benchmarks provide instruction-following data suitable for training LMMs, they lack pairwise preference annotations required for DPO. To address this gap and enable continual learning via DPO, we construct pairwise preference data for each dataset within these three benchmarks.
In particular, for each instruction instance, we treat the provided reference answer as the preferred output $y^+$, which represents a well-retained (good-memory) and well-adapted response. To simulate the less preferred (forgotten) output $y^-$, we prompt a large language model to hallucinate an alternative response. 
The model is conditioned on both the textual instruction and the reference answer $y^+$, and is instructed to generate an output  $y^-$ that is plausible and coherent, yet distinct from $y^+$ and potentially flawed in subtle ways. This design encourages the formation of challenging preference pairs suitable for effective DPO training.
Finally, all labels are manually verified by human annotators to ensure that the rejected responses accurately reflect undesirable or suboptimal behavior (Figure \ref{fig:example-dpo}).

\vspace{-2mm}
\section{Experimental Results}
\label{sec:exp}

\subsection{Benchmarks, Metrics, and Implementation}

\noindent
\textbf{Benchmark and Metrics.} We evaluate on three benchmarks: CoIN, MLLM-CL Domain, and MLLM-CL Ability. Following standard protocols \cite{zhao2025mllm, guo2025hidellava, chen2024coin}, continual learning performance is measured using five metrics. \textbf{Last Accuracy} reports accuracy on all seen tasks after learning the final one. \textbf{Mean Finetune Accuracy} (MFT) reflects accuracy on each task immediately after learning, serving as an upper bound without forgetting. \textbf{Mean Final Accuracy} (MFN) averages accuracy across all tasks after full training. \textbf{Mean Average Accuracy} (MAA) captures the average performance over all tasks after each step. \textbf{Backward Transfer} (BWT) quantifies forgetting by comparing final accuracy to post-learning accuracy for each task.
Higher scores indicate better performance.

\noindent
\textbf{Implementation.} 
Our framework adopts the implementation of LLaVA v1.5 \cite{liu2024improved}, using CLIP-ViT-L-14 ($336^2$) as the vision encoder and Vicuna 7B \cite{vicuna2023} as the language backbone. For fair comparison, we follow the training setup of \cite{zhao2025mllm, guo2025hidellava, chen2024coin}, with LoRA rank 32, AdamW optimizer, a cosine learning rate schedule (base LR: $2\mathrm{e}{-5}$), and batch size 64 for one training epoch. All experiments are run on 16 NVIDIA A100 40GB GPUs.

\subsection{Main Results}

\noindent
\textbf{Results on the MLLM-CL Domain Benchmark.} 
Table~\ref{tab:main-domain} shows results on the MLLM-CL Domain benchmark across five incremental domains: Remote Sensing (RS), Medical, Autonomous Driving (AD), Science, and Finance.
Our \method achieves consistently outperforms prior methods in both individual accuracy (on last step) and continual learning metrics. In particular, 
\method achieves 85.58\% on RS, 95.28\% on Finance, and shows similar improvements on Medical (69.74\%), AD (57.75\%), and Science (61.55\%).
The results has indicated the robustness of our \method under domain shifts. In terms of continual learning performance, \method achieves an MFT of 74.29\%, an MFN of 74.00\%, and an MAA of 75.78\%. The BWT is -0.37\%, further confirming the ability for \method in forgetting mitigation across incremental domain adaptation. Compared to prior methods that rely on LoRA (\eg, MR-LoRA) or Mixture-of-Expert architectures (\eg, CL-MoE), \method offers a more unified framework while gaining superior results. These findings reinforce the generalizability and stability of \method in continual domain-incremental learning.

\begin{table}[!t]
\setlength{\tabcolsep}{3pt}
\caption{Results on MLLM-CL Domain (* denote the method using relay data). RS: Remote Sensing, Med: Medical, AD: Autonmous Driving, Sci: ScienceQA, Fin: Finance.}\label{tab:main-domain}
\vspace{-3mm}
\resizebox{1.0\linewidth}{!}{
\begin{tabular}{l|ccccc|cccc}
\hline
Method   & RS    & Med   & AD    & Sci   & Fin   & MFT$\uparrow$ & MFN$\uparrow$ & MAA$\uparrow$  & BWT$\uparrow$   \\
\hline
Zeroshot & 32.29 & 28.28 & 15.59 & 35.55 & 62.56 & 34.85 &  $-$  & $-$   & $-$    \\
\hline
LoRA-FT* \cite{hu2022lora} & 76.54 & 50.27 & 43.01 & 43.32 & 89.85 & 66.32 & 60.60 & 64.72 & -7.15  \\
O-LoRA* \cite{wang2023orthogonal}  & 76.94 & 41.17 & 34.18 & 39.61 & 83.22 & 60.49 & 55.02 & 60.73 & -6.83  \\
MoELoRA* \cite{chen2024coin} & 77.63 & 49.54 & 39.08 & 41.04 & 89.21 & 66.24 & 59.30 & 64.81 & -8.68  \\
CL-MoE* \cite{huai2025cl}  & 76.58 & 52.31 & 39.65 & 45.64 & 90.21 & 66.65 & 60.88 & 64.95 & -7.22  \\
HiDe* \cite{guo2025hidellava}    & 74.80 & 42.29 & 34.03 & 38.01 & 79.22 & 60.83 & 53.67 & 61.81 & -8.95  \\
SEFE* \cite{chen2025sefe}    & 78.43 & 52.85 & 46.21 & 47.76 & 89.33 & 66.89 & 62.92 & 66.51 & -4.97  \\
DISCO* \cite{guo2025federated}  & 77.78 & 46.25 & 50.45 & 49.51 & 89.71 & 65.27 & 62.74 & 64.92 & -3.17  \\
\hline
LoRA-FT \cite{hu2022lora} & 69.65 & 41.59 & 25.43 & 40.88 & 87.45 & 64.98 & 53.00 & 61.13 & -14.97 \\
O-LoRA \cite{wang2023orthogonal}  & 74.64 & 44.42 & 30.02 & 41.47 & 87.15 & 65.16 & 55.54 & 62.12 & -12.03 \\
MoELoRA \cite{chen2024coin} & 77.54 & 41.85 & 27.62 & 40.13 & 86.75 & 64.94 & 54.78 & 61.76 & -12.70 \\
CL-MoE \cite{huai2025cl}  & 71.34 & 46.84 & 26.33 & 41.17 & 88.74 & 66.06 & 54.88 & 61.79 & -13.96 \\
HiDe \cite{guo2025hidellava}    & 74.31 & 48.95 & 33.21 & 38.54 & 81.55 & 60.77 & 55.31 & 60.68 & -6.82  \\
SEFE \cite{chen2025sefe}     & 77.26 & 50.37 & 37.21 & 40.87 & 86.82 & 65.01 & 58.51 & 63.63 & -8.13  \\
DISCO    \cite{guo2025federated} & 76.03 & 45.20 & 43.79 & 42.33 & 88.95 & 64.43 & 59.26 & 63.35 & -6.46  \\
MR-LoRA \cite{zhao2025mllm} & 80.87 & 65.32 & 54.12 & 56.71 & 91.12 & 69.64 & 69.63 & 71.06 & \textbf{-0.01}  \\
\hline
\textbf{\method} & \textbf{85.68} & \textbf{69.74} & \textbf{57.73} & \textbf{61.55} & \textbf{95.28} & \textbf{74.29} & \textbf{74.00} & \textbf{75.68} & {-0.37} \\
\hline
\end{tabular}
}
\vspace{-6mm}
\end{table}

\noindent
\textbf{Results on the MLLM-CL Ability Benchmark.} 
Table~\ref{tab:main-ability} shows results on the MLLM-CL Ability benchmark across four incremental tasks: OCR, Math \& Logic (M\&L), Visual Programming (VP), and GUI Agent (GUI), with the first four columns reporting performance after the final task.
Our \method achieves consistently strong results across all tasks, with 38.40\% on OCR, 39.20\% on M\&L, 68.65\% on VP, and 35.00\% on GUI Agent. In addition, \method outperforms prior methods on all continual learning metrics. \method achieves a MFT of 45.55\%, MFN of 45.31\%, and MAA of 43.03, while gaining a BWT of -0.31\%, indicating minimal forgetting. Compared to DISCO and MR-LoRA, which rely on LoRA and routing strategies, \method demonstrates improved overall performance. These results highlight the effectiveness of \method in maintaining stability and generalization throughout the incremental learning process.

\begin{table}[!t]
\caption{Results on MLLM-CL Ability (* denote the method using relay data). M\&L: Math and Logic, VP: Visual Perception.}\label{tab:main-ability}
\setlength{\tabcolsep}{3pt}
\vspace{-3mm}
\resizebox{1.0\linewidth}{!}{
\begin{tabular}{l|cccc|cccc}
\hline
Method       & OCR   & M\&L   & VP    & GUI & MFT$\uparrow$  & MFN$\uparrow$  & MAA$\uparrow$  & BWT$\uparrow$   \\
\hline
Zeroshot \cite{hu2022lora}    & 31.20 & 30.20 & 60.79 & 10.00     & 33.05 &  $-$  & $-$   & $-$    \\
\hline
LoRA-FT* \cite{hu2022lora}    & 21.80 & 32.70 & 58.38 & 28.75     & 40.32 & 35.41 & 36.32 & -6.55  \\
O-LoRA*  \cite{wang2023orthogonal}    & 29.60 & 31.30 & 60.79 & 27.50     & 39.96 & 37.30 & 36.34 & -3.55  \\
MoELoRA* \cite{chen2024coin}    & 19.80 & 32.20 & 54.19 & 30.00     & 40.35 & 34.05 & 35.39 & -8.41  \\
CL-MoE* \cite{huai2025cl}    & 25.40 & 31.80 & 60.91 & 30.00     & 41.22 & 37.03 & 37.28 & -5.59  \\
HiDe*  \cite{guo2025hidellava}      & 24.60 & 28.40 & 30.71 & 23.75     & 36.84 & 26.86 & 33.54 & -13.30 \\
SEFE*  \cite{chen2025sefe}      & 25.60 & 34.80 & 57.61 & 31.39     & 42.25 & 37.35 & 37.93 & -6.53  \\
DISCO* \cite{huai2025cl}      & 34.20 & 35.00 & 61.55 & 27.50     & 40.14 & 39.56 & 37.85 & -0.77  \\
\hline
LoRA-FT \cite{hu2022lora}     & 23.60 & 33.70 & 55.84 & 32.50     & 41.28 & 36.41 & 36.58 & -6.49  \\
O-LoRA \cite{wang2023orthogonal}      & 29.60 & 32.90 & 52.41 & 33.75     & 39.72 & 37.16 & 35.42 & -3.41  \\
MoELoRA \cite{chen2024coin}     & 26.70 & 32.80 & 56.85 & 27.22     & 39.45 & 35.89 & 36.07 & -4.75  \\
CL-MoE \cite{huai2025cl}      & 19.90 & 32.70 & 53.43 & 30.69     & 40.50 & 34.18 & 35.65 & -8.43  \\
HiDe  \cite{guo2025hidellava}       & 24.60 & 32.10 & 46.32 & 28.75     & 37.98 & 32.94 & 34.60 & -6.72  \\
SEFE \cite{chen2025sefe}        & 26.00 & 33.40 & 57.74 & 33.75     & 40.98 & 37.72 & 36.59 & -4.35  \\
DISCO  \cite{guo2025federated}      & 32.90 & 33.10 & 60.15 & 30.14     & 39.02 & 39.07 & 36.57 & 0.07   \\
MR-LoRA  \cite{zhao2025mllm}    & 33.70 & 36.20 & 65.10 & 32.50     & 41.89 & 41.88 & 38.86 & \textbf{-0.02}  \\
\hline
\textbf{\method} & \textbf{38.40} & \textbf{39.20} & \textbf{68.65} & \textbf{35.00} & \textbf{45.55} & \textbf{45.31} & \textbf{43.03} & -0.31 \\
\hline
\end{tabular}
}
\vspace{-4mm}
\end{table}

\begin{table}[!b]
\vspace{-4mm}
\caption{Results on CoIN. SciQA: ScienceQA, Image: ImageNet, Viz: VizWiz, Ground: Grounding, Text: TextVQA, VQA: VQAv2.}\label{tab:main-coin}
\setlength{\tabcolsep}{2pt}
\vspace{-3mm}
\resizebox{1.0\linewidth}{!}{
\begin{tabular}{l|cccccccc|cc}
\hline
Method     & SciQA          & Image          & Viz            & Ground      & Text           & GQA            & VQA            & OCR            & MFN$\uparrow$            & MAA$\uparrow$            \\
\hline
Zeroshot   & 69.79          & 9.93           & 45.50          & 58.47          & 57.75          & 60.77          & 66.50          & 64.93      &   $-$    &  $-$              \\
FineTune   & 57.43          & 28.90          & 41.88          & 30.05          & 51.39          & 50.76          & 53.28          & 64.78          & 47.31          & 52.86          \\
LwF \cite{li2017learning}        & 60.71          & 30.58          & 41.49          & 36.01          & 52.80          & 47.07          & 53.43          & 65.12          & 48.40          & 53.22          \\
EWC \cite{kirkpatrick2017overcoming}        & 59.75          & 31.88          & 42.26          & 34.96          & 51.06          & 51.84          & 55.30          & 64.55          & 48.95          & 53.30          \\
L2P \cite{wang2022learning}        & 70.21          & 23.31          & 44.21          & 43.76          & 56.25          & 58.46          & 62.32          & 64.11          & 52.83          & 53.96          \\
O-LoRA \cite{wang2023orthogonal}     & 72.56          & 62.84          & 48.43          & 58.97          & 57.66          & 59.14          & 63.21          & 63.31          & 60.77          & 62.60          \\
MoELoRA  \cite{chen2024coin}  & 62.02          & 37.21          & 43.32          & 33.22          & 52.05          & 53.12          & 57.92          & 65.75          & 50.58          & 55.24          \\
HiDe \cite{guo2025hidellava} & 73.20          & 69.28          & 50.76          & 59.18          & 56.92          & 61.33          & 67.12          & 64.76          & 62.82          & 64.70          \\
\hline
\textbf{\method}       & \textbf{77.84} & \textbf{95.61} & \textbf{54.55} & \textbf{60.74} & \textbf{59.17} & \textbf{64.32} & \textbf{69.99} & \textbf{68.69} & \textbf{68.86} & \textbf{74.94} \\
\hline
\end{tabular}
}
\end{table}

\noindent
\textbf{Results on the CoIN Benchmark.} 
Table~\ref{tab:main-coin} shows results on the CoIN benchmark. MFT and BWT are omitted for prior methods since intermediate models were not published. The first eight columns report final-task accuracy on: ScienceQA (SciQA), ImageNet (Image), VizWiz (Viz), Grounding (Ground), TextVQA (Text), GQA, VQAv2 (VQA), and OCR.
Our \method consistently outperforms all prior methods across these tasks. In particular, it achieves major gains on vision-centric tasks, \ie, ImageNet (95.61\%), VizWiz (54.55\%), and OCR (68.69\%), while also obtaining strong results on language and reasoning benchmarks, \ie, ScienceQA (77.84\%), GQA (64.32\%), and VQAv2 (69.99\%). Importantly, our approach maintains a high Mean Final Accuracy (MFN) of 68.86\%, reflecting superior knowledge preservation across incremental tasks. \method further achieves the best MAA of 74.94\%, demonstrating stable performance throughout continual training. These results confirm the effectiveness of \method in mitigating catastrophic forgetting and increasing adaptability.

\subsection{Ablation Study}

\noindent
\textbf{Effectiveness of Fairness DPO.}
Table~\ref{tab:ablation-dpo} presents an ablation study on the impact of each component in our approach. Compared to knowledge distillation (KD), vanilla DPO achieves consistently better performance across incremental domains, \ie RS (82.26\%), Med (67.12\%), AD (55.79\%), Sci (59.58\%), and Fin (93.94\%), as well as improved continual learning metrics: MFT of 72.80\%, MFN of 71.74\%, MAA of 73.68\%, and BWT of -1.33\%, indicating reduced forgetting. Our \method further improves the performance, \ie 74.29\% MFT, 74.00\% MFN, 75.68\% MAA, and -0.37\% BWT, demonstrating stronger task performance and greater stability. These results highlight the effectiveness of our Fair DPO loss in continual learning.

\begin{table}[!t]
\caption{Effectiveness of Our Fairness DPO.}\label{tab:ablation-dpo}
\vspace{-3mm}
\setlength{\tabcolsep}{3pt}
\resizebox{1.0\linewidth}{!}{
\begin{tabular}{l|ccccc|cccc}
\hline
             & RS & Med & AD & Sci & Fin & MFT$\uparrow$  & MFN$\uparrow$  & MAA$\uparrow$  & BWT$\uparrow$   \\
\hline
Zeroshot & 32.29 & 28.28 & 15.59 & 35.55 & 62.56 & 34.85 &  $-$  & $-$   & $-$    \\
LoRA-FT  & 69.65 & 41.59 & 25.43 & 40.88 & 87.45 & 64.98 & 53.00 & 61.13 & -14.97 \\
KD           &   77.82	& 63.92	& 52.88	& 58.57	& 92.85	& 71.37	& 69.21	& 71.49	& -2.71 \\
\hline
DPO  &  82.26 & 67.12 & 55.79 & 59.58 & 93.94 & 72.80 & 71.74 & 73.68 & -1.33 \\
\method &  \textbf{85.68} & \textbf{69.74} & \textbf{57.73} & \textbf{61.55} & \textbf{95.28} & \textbf{74.29} & \textbf{74.00} & \textbf{75.68} & \textbf{-0.37} \\
\hline
\end{tabular}
}
\vspace{-6mm}
\end{table}

\noindent
\textbf{Effectiveness of Divergence Parameter $\beta$.} 
Table \ref{tab:ablation-beta} studies the impact of the divergence parameter $\beta$ in our approach. The $\beta$ parameter controls the trade-off between adaptability to new tasks and the forgetting level of previous knowledge. As shown in Table \ref{tab:ablation-beta}, lower values of $\beta$ (\ie, $0.01$\ and $0.05$) lead to faster adaptation, reflected in higher MFT scores (75.43\% and 75.27\%). However, it will result in increased forgetting, as indicated by degraded BWT values (-2.66\% and -1.66\%). Meanwhile, larger values of $\beta$ (\ie, $0.50$) reduce forgetting with improved BWT (-0.32\%), but at the cost of reduced overall performance, including lower MFT (72.11\%), MFN (71.85\%), and MAA (73.60\%). Among our configurations, $\beta = 0.10$ achieves the best balance, with competitive MFT (74.29\%), MFN (74.00\%), and MAA (75.68\%), while maintaining a favorable BWT of -0.37\%. These results have illustrated the role of $\beta$ in stability and plasticity in continual learning.

\begin{table}[!b]
\vspace{-6mm}
\caption{Effectiveness of Divergence Parameter $\beta$.}\label{tab:ablation-beta}
\vspace{-3mm}
\setlength{\tabcolsep}{3pt}
\resizebox{1.0\linewidth}{!}{
\begin{tabular}{c|ccccc|cccc}
\hline
$\beta$ & RS & Med & AD & Sci & Fin & MFT$\uparrow$  & MFN$\uparrow$  & MAA$\uparrow$  & BWT$\uparrow$   \\
\hline
0.01 & 83.09          & 67.71          & 56.90          & 62.06          & \textbf{96.77} & 75.43          & 73.31          & 75.83          & -2.66          \\
0.05 & 84.91          & 69.30          & 57.16          & \textbf{62.21} & 96.12          & \textbf{75.27} & 73.94          & \textbf{76.27} & -1.66          \\
0.10 & \textbf{85.68} & \textbf{69.74} & \textbf{57.73} & 61.55          & 95.28          & 74.29          & \textbf{74.00} & 75.68          & -0.37          \\
0.50 & 83.74          & 67.55          & 55.44          & 59.64          & 92.88          & 72.11          & 71.85          & 73.60          & \textbf{-0.32} \\
\hline
\end{tabular}
}
\end{table}

\noindent
\textbf{Effectiveness of Focusing Parameter $\gamma$.}
Table~\ref{tab:ablation-gamma} presents an ablation on the focusing parameter $\gamma$, which controls the emphasis on harder preference pairs during training.
When $\gamma = 0.0$, the loss reduces to the standard DPO formulation, corresponding to the vanilla DPO baseline in Table \ref{tab:ablation-dpo}. Lower values of $\gamma$ (\ie, $0.50$ and $1.00$) yield moderate improvements in stability and forgetting, as reflected in reduced BWT (-1.07\% and -0.77\%) while maintaining competitive MFT (72.79\% and 73.38\%) and MAA (73.87\% and 74.66\%). This result suggests that incorporating moderate focus on harder examples enhances knowledge preservation without compromising adaptability. Meanwhile, the high values of $\gamma$ (\eg, $5.00$) lead to degraded performance across most metrics, with reduced MFT (73.21\%), MFN (72.18\%), and MAA (74.27\%), indicating that overemphasizing difficult pairs can lower overall learning dynamics. In our experiments, $\gamma = 2.00$ provides the best trade-off, achieving favorable MFT (74.29\%), MFN (74.00\%), MAA (75.68\%), and BWT (-0.37\%). These results have indicated the importance of focusing parameter $\gamma$ in balancing plasticity and stability in our \method approach.

\begin{table}[!t]
\caption{Effectiveness of Focusing Parameter $\gamma$.}\label{tab:ablation-gamma}
\vspace{-3mm}
\setlength{\tabcolsep}{3pt}
\resizebox{1.0\linewidth}{!}{
\begin{tabular}{c|ccccc|cccc}
\hline
$\gamma$ & RS & Med & AD & Sci & Fin & MFT$\uparrow$  & MFN$\uparrow$  & MAA$\uparrow$  & BWT$\uparrow$   \\
\hline
0.00 & 82.26          & 67.12          & 55.79          & 59.58          & 93.94          & 72.80          & 71.74          & 73.68          & -1.33          \\
0.50 & 82.84          & 67.84          & 55.50          & 59.71          & 93.78          & 72.79          & 71.93          & 73.87          & -1.07          \\
1.00 & 84.01          & 68.44          & 56.56          & 60.74          & 94.06          & 73.38          & 72.76          & 74.66          & -0.77          \\
2.00 & \textbf{85.68} & \textbf{69.74} & \textbf{57.73} & \textbf{61.55} & \textbf{95.28} & \textbf{74.29} & \textbf{74.00} & \textbf{75.68} & \textbf{-0.37} \\
5.00 & 83.08          & 67.74          & 55.99          & 60.05          & 94.03          & 73.21          & 72.18          & 74.27          & -1.29  \\
\hline
\end{tabular}
}
\vspace{-6mm}
\end{table}

\noindent
\textbf{Effectiveness of Different LMMs.}
Table~\ref{tab:ablation-llm} evaluates \method across different LMMs: LLaVA-7B, LLaVA-13B, and InternVL-7B \cite{chen2024internvl}. In all cases, \method outperforms standard DPO, confirming the generality of our Fair DPO objective. Larger models (\ie LLaVA-13B) show improved performance over LLaVA-7B in MFT (76.29\% vs. 74.29\%), MFN (75.81\% vs. 74.00\%), MAA (77.57\% vs. 75.68\%), and BWT (-0.59\% vs. -0.37\%), reflecting better knowledge preservation due to greater capacity. InternVL-7B, despite similar size to LLaVA-7B, benefits from stronger vision-language alignment and achieves competitive results. \method consistently enhances performance across all backbones, demonstrating robustness and compatibility with diverse LMM architectures.

\begin{table}[!b]
\vspace{-6mm}
\caption{Effectiveness of Different LMM Framework.}\label{tab:ablation-llm}
\setlength{\tabcolsep}{3pt}
\vspace{-3mm}
\resizebox{1.0\linewidth}{!}{
\begin{tabular}{ll|ccccc|cccc}
\hline
LLM     & Method       & RS & Med & AD & Sci & Fin & MFT$\uparrow$  & MFN$\uparrow$  & MAA$\uparrow$  & BWT$\uparrow$   \\
\hline
\multirow{2}{*}{LLaVA-7B}
    & DPO  &  82.26 & 67.12 & 55.79 & 59.58 & 93.94 & 72.80 & 71.74 & 73.68 & -1.33 \\
    & \method & 
    \textbf{85.68} & \textbf{69.74} & \textbf{57.73} & \textbf{61.55} & \textbf{95.28} & \textbf{74.29} & \textbf{74.00} & \textbf{75.68} & \textbf{-0.37} \\
\hline
\multirow{2}{*}{LLaVA-13B}
    & DPO  &  86.24          & 70.28          & 58.02          & 63.26          & 96.84          & 75.89          & 74.93          & 77.00          & -1.21          \\
    & \method & \textbf{87.40} & \textbf{71.09} & \textbf{59.34} & \textbf{63.96} & \textbf{97.28} & \textbf{76.29} & \textbf{75.81} & \textbf{77.57} & \textbf{-0.59} \\
\hline
\multirow{2}{*}{InternVL-7B}
    & DPO  &  82.69          & 67.69          & 55.97          & 60.67          & 94.83          & 73.88          & 72.37          & 74.68          & -1.89          \\
    & \method & \textbf{85.64} & \textbf{69.88} & \textbf{57.94} & \textbf{61.95} & \textbf{95.74} & \textbf{74.62} & \textbf{74.23} & \textbf{75.94} & \textbf{-0.49} \\
\hline
\end{tabular}
}
\end{table}

\vspace{-2mm}
\section{Conclusions and Limitations}
\label{sec:concl}

\textbf{Conclusions.} This paper has presented a novel Fairness DPO approach to Continual Learning in LMMs. 
In particular, our Fair DPO learning objective has been introduced to address both catastrophic forgetting and fairness problems. Our theoretical analysis has also shown the effectiveness of our proposed approach.
Our SoTA results on three  benchmarks have further confirmed the effectiveness of our \method compared to prior methods.

\noindent
\textbf{Limitations.} 
Our work adopts a set of learning hyper-parameters aligned with theoretical analysis, but this choice introduces limitations, particularly in tuning $\beta$, $\gamma$, the weighted loss in Eqn.~\eqref{eqn:final-cl-mllm}, and DPO data construction. The quality of DPO data is sensitive to label stability, which may be affected by class imbalance, domain shifts, or model uncertainty, potentially leading to suboptimal distillation targets and misleading pairwise correlations. These limitations highlight the need for future work on more robust and adaptive DPO strategies for continual multimodal learning.

\small{
\noindent
\textbf{Acknowledgment.} 
This work is partly supported by NSF CAREER (No. 2442295), NSF SCH (No. 2501021), NSF E-RISE (No. 2445877), NSF BIO (No. 2524623) and USDA/NIFA Award. We also acknowledge the Arkansas High-Performance Computing Center (HPC) for GPU servers.
}

{
    \small
    \bibliographystyle{ieeenat_fullname}
    \bibliography{references}

\begin{thebibliography}{124}
\providecommand{\natexlab}[1]{#1}
\providecommand{\url}[1]{\texttt{#1}}
\expandafter\ifx\csname urlstyle\endcsname\relax
  \providecommand{\doi}[1]{doi: #1}\else
  \providecommand{\doi}{doi: \begingroup \urlstyle{rm}\Url}\fi

\bibitem[Acharya et~al.(2019)Acharya, Kafle, and Kanan]{acharya2019tallyqa}
Manoj Acharya, Kushal Kafle, and Christopher Kanan.
\newblock Tallyqa: Answering complex counting questions.
\newblock In \emph{Proceedings of the AAAI conference on artificial intelligence}, pages 8076--8084, 2019.

\bibitem[Achiam et~al.(2023)Achiam, Adler, Agarwal, Ahmad, Akkaya, Aleman, Almeida, Altenschmidt, Altman, Anadkat, et~al.]{achiam2023gpt}
Josh Achiam, Steven Adler, Sandhini Agarwal, Lama Ahmad, Ilge Akkaya, Florencia~Leoni Aleman, Diogo Almeida, Janko Altenschmidt, Sam Altman, Shyamal Anadkat, et~al.
\newblock Gpt-4 technical report.
\newblock \emph{arXiv}, 2023.

\bibitem[Aguilar et~al.(2020)Aguilar, Ling, Zhang, Yao, Fan, and Guo]{aguilar2020knowledge}
Gustavo Aguilar, Yuan Ling, Yu Zhang, Benjamin Yao, Xing Fan, and Chenlei Guo.
\newblock Knowledge distillation from internal representations.
\newblock In \emph{Proceedings of the AAAI conference on artificial intelligence}, pages 7350--7357, 2020.

\bibitem[Alayrac et~al.(2022)Alayrac, Donahue, Luc, Miech, Barr, Hasson, Lenc, Mensch, Millican, Reynolds, et~al.]{alayrac2022flamingo}
Jean-Baptiste Alayrac, Jeff Donahue, Pauline Luc, Antoine Miech, Iain Barr, Yana Hasson, Karel Lenc, Arthur Mensch, Katherine Millican, Malcolm Reynolds, et~al.
\newblock Flamingo: a visual language model for few-shot learning.
\newblock \emph{Advances in neural information processing systems}, 35:\penalty0 23716--23736, 2022.

\bibitem[Bai et~al.(2023{\natexlab{a}})Bai, Bai, Chu, Cui, Dang, Deng, Fan, Ge, Han, Huang, et~al.]{bai2023qwen}
Jinze Bai, Shuai Bai, Yunfei Chu, Zeyu Cui, Kai Dang, Xiaodong Deng, Yang Fan, Wenbin Ge, Yu Han, Fei Huang, et~al.
\newblock Qwen technical report.
\newblock \emph{arXiv}, 2023{\natexlab{a}}.

\bibitem[Bai et~al.(2023{\natexlab{b}})Bai, Bai, Yang, Wang, Tan, Wang, Lin, Zhou, and Zhou]{bai2023qwenvl}
Jinze Bai, Shuai Bai, Shusheng Yang, Shijie Wang, Sinan Tan, Peng Wang, Junyang Lin, Chang Zhou, and Jingren Zhou.
\newblock Qwen-vl: A versatile vision-language model for understanding, localization, text reading, and beyond.
\newblock \emph{arXiv}, 2023{\natexlab{b}}.

\bibitem[Buzzega et~al.(2020)Buzzega, Boschini, Porrello, Abati, and Calderara]{buzzega2020dark}
Pietro Buzzega, Matteo Boschini, Angelo Porrello, Davide Abati, and Simone Calderara.
\newblock Dark experience for general continual learning: a strong, simple baseline.
\newblock \emph{Advances in neural information processing systems}, 33:\penalty0 15920--15930, 2020.

\bibitem[Cao et~al.(2024)Cao, Liu, Liu, Wang, Dong, Ding, Zhang, Reid, and Liang]{cao2024continualllava}
Meng Cao, Yuyang Liu, Yingfei Liu, Tiancai Wang, Jiahua Dong, Henghui Ding, Xiangyu Zhang, Ian Reid, and Xiaodan Liang.
\newblock Continual llava: Continual instruction tuning in large vision-language models.
\newblock \emph{arXiv}, 2024.

\bibitem[Cattiaux and Guillin(2003)]{cattiaux2003criterion}
Patrick Cattiaux and Arnaud Guillin.
\newblock A criterion for talagrand's quadratic transportation cost inequality.
\newblock \emph{arXiv}, 2003.

\bibitem[Cermelli et~al.(2020)Cermelli, Mancini, Rota~Bulò, Ricci, and Caputo]{cermelli2020modelingthebackground}
Fabio Cermelli, Massimiliano Mancini, Samuel Rota~Bulò, Elisa Ricci, and Barbara Caputo.
\newblock Modeling the background for incremental learning in semantic segmentation.
\newblock In \emph{CVPR}, 2020.

\bibitem[Chang et~al.(2022)Chang, Palzer, Li, Fosler-Lussier, and Xiao]{chang2022mapqa}
Shuaichen Chang, David Palzer, Jialin Li, Eric Fosler-Lussier, and Ningchuan Xiao.
\newblock Mapqa: A dataset for question answering on choropleth maps.
\newblock \emph{arXiv}, 2022.

\bibitem[Chang et~al.(2025)Chang, Chang, and Wu]{chang2025balora}
Yupeng Chang, Yi Chang, and Yuan Wu.
\newblock Ba-lora: Bias-alleviating low-rank adaptation to mitigate catastrophic inheritance in large language models, 2025.

\bibitem[Chen et~al.(2024{\natexlab{a}})Chen, Zhu, Luo, Shen, Song, and Gao]{chen2024coin}
Cheng Chen, Junchen Zhu, Xu Luo, Heng~T Shen, Jingkuan Song, and Lianli Gao.
\newblock Coin: A benchmark of continual instruction tuning for multimodel large language models.
\newblock \emph{Advances in Neural Information Processing Systems}, 37:\penalty0 57817--57840, 2024{\natexlab{a}}.

\bibitem[Chen and Zhang(2024)]{chenfedmbridge}
Jiayi Chen and Aidong Zhang.
\newblock Fed{MB}ridge: Bridgeable multimodal federated learning.
\newblock In \emph{Forty-first International Conference on Machine Learning}, 2024.

\bibitem[Chen et~al.(2025)Chen, Cong, Zhao, Yang, Hu, Ip, and Kwong]{chen2025sefe}
Jinpeng Chen, Runmin Cong, Yuzhi Zhao, Hongzheng Yang, Guangneng Hu, Horace Ho~Shing Ip, and Sam Kwong.
\newblock Sefe: Superficial and essential forgetting eliminator for multimodal continual instruction tuning.
\newblock \emph{arXiv}, 2025.

\bibitem[Chen et~al.(2024{\natexlab{b}})Chen, Wu, Wang, Su, Chen, Xing, Zhong, Zhang, Zhu, Lu, et~al.]{chen2024internvl}
Zhe Chen, Jiannan Wu, Wenhai Wang, Weijie Su, Guo Chen, Sen Xing, Muyan Zhong, Qinglong Zhang, Xizhou Zhu, Lewei Lu, et~al.
\newblock Internvl: Scaling up vision foundation models and aligning for generic visual-linguistic tasks.
\newblock In \emph{Proceedings of the IEEE/CVF conference on computer vision and pattern recognition}, pages 24185--24198, 2024{\natexlab{b}}.

\bibitem[Chiang et~al.(2023)Chiang, Li, Lin, Sheng, Wu, Zhang, Zheng, Zhuang, Zhuang, Gonzalez, Stoica, and Xing]{vicuna2023}
Wei-Lin Chiang, Zhuohan Li, Zi Lin, Ying Sheng, Zhanghao Wu, Hao Zhang, Lianmin Zheng, Siyuan Zhuang, Yonghao Zhuang, Joseph~E. Gonzalez, Ion Stoica, and Eric~P. Xing.
\newblock Vicuna: An open-source chatbot impressing gpt-4 with 90\%* chatgpt quality, 2023.

\bibitem[Cossu et~al.(2024)Cossu, Carta, Passaro, Lomonaco, Tuytelaars, and Bacciu]{cossu2024continual}
Andrea Cossu, Antonio Carta, Lucia Passaro, Vincenzo Lomonaco, Tinne Tuytelaars, and Davide Bacciu.
\newblock Continual pre-training mitigates forgetting in language and vision.
\newblock \emph{Neural Networks}, 179:\penalty0 106492, 2024.

\bibitem[Deitke et~al.(2025)Deitke, Clark, Lee, Tripathi, Yang, Park, Salehi, Muennighoff, Lo, Soldaini, et~al.]{deitke2025molmo}
Matt Deitke, Christopher Clark, Sangho Lee, Rohun Tripathi, Yue Yang, Jae~Sung Park, Mohammadreza Salehi, Niklas Muennighoff, Kyle Lo, Luca Soldaini, et~al.
\newblock Molmo and pixmo: Open weights and open data for state-of-the-art vision-language models.
\newblock In \emph{Proceedings of the Computer Vision and Pattern Recognition Conference}, pages 91--104, 2025.

\bibitem[Deng et~al.(2009)Deng, Dong, Socher, Li, Li, and Fei-Fei]{deng2009imagenet}
Jia Deng, Wei Dong, Richard Socher, Li-Jia Li, Kai Li, and Li Fei-Fei.
\newblock Imagenet: A large-scale hierarchical image database.
\newblock In \emph{2009 IEEE conference on computer vision and pattern recognition}, pages 248--255. Ieee, 2009.

\bibitem[Douillard et~al.(2021)Douillard, Chen, Dapogny, and Cord]{douillard2021plop}
Arthur Douillard, Yifu Chen, Arnaud Dapogny, and Matthieu Cord.
\newblock Plop: Learning without forgetting for continual semantic segmentation.
\newblock In \emph{Proceedings of the IEEE/CVF Conference on Computer Vision and Pattern Recognition}, pages 4040--4050, 2021.

\bibitem[Douillard et~al.(2022)Douillard, Ram{\'e}, Couairon, and Cord]{douillard2022dytox}
Arthur Douillard, Alexandre Ram{\'e}, Guillaume Couairon, and Matthieu Cord.
\newblock Dytox: Transformers for continual learning with dynamic token expansion.
\newblock In \emph{Proceedings of the IEEE/CVF conference on computer vision and pattern recognition}, pages 9285--9295, 2022.

\bibitem[Edwards(2011)]{edwards2011kantorovich}
David~A Edwards.
\newblock On the kantorovich--rubinstein theorem.
\newblock \emph{Expositiones Mathematicae}, 29\penalty0 (4):\penalty0 387--398, 2011.

\bibitem[Ghosal et~al.(2023)Ghosal, Majumder, Mehrish, and Poria]{ghosal2023text}
Deepanway Ghosal, Navonil Majumder, Ambuj Mehrish, and Soujanya Poria.
\newblock Text-to-audio generation using instruction-tuned llm and latent diffusion model.
\newblock \emph{arXiv}, 2023.

\bibitem[Goyal et~al.(2017)Goyal, Khot, Summers-Stay, Batra, and Parikh]{goyal2017making}
Yash Goyal, Tejas Khot, Douglas Summers-Stay, Dhruv Batra, and Devi Parikh.
\newblock Making the v in vqa matter: Elevating the role of image understanding in visual question answering.
\newblock In \emph{Proceedings of the IEEE conference on computer vision and pattern recognition}, pages 6904--6913, 2017.

\bibitem[Guo et~al.(2025{\natexlab{a}})Guo, Zeng, Xiang, Zhu, Wang, Zhang, and Liu]{guo2025hidellava}
Haiyang Guo, Fanhu Zeng, Ziwei Xiang, Fei Zhu, Da-Han Wang, Xu-Yao Zhang, and Cheng-Lin Liu.
\newblock Hide-llava: Hierarchical decoupling for continual instruction tuning of multimodal large language model.
\newblock \emph{arXiv}, 2025{\natexlab{a}}.

\bibitem[Guo et~al.(2025{\natexlab{b}})Guo, Zeng, Zhu, Liu, Wang, Xu, Zhang, and Liu]{guo2025federated}
Haiyang Guo, Fanhu Zeng, Fei Zhu, Wenzhuo Liu, Da-Han Wang, Jian Xu, Xu-Yao Zhang, and Cheng-Lin Liu.
\newblock Federated continual instruction tuning.
\newblock \emph{arXiv}, 2025{\natexlab{b}}.

\bibitem[Guo et~al.(2025{\natexlab{c}})Guo, Zeng, Zhu, Wang, Wang, Zhou, Zhao, Liu, Ma, Zhang, et~al.]{guo2025comprehensive}
Haiyang Guo, Fanhu Zeng, Fei Zhu, Jiayi Wang, Xukai Wang, Jingang Zhou, Hongbo Zhao, Wenzhuo Liu, Shijie Ma, Xu-Yao Zhang, et~al.
\newblock A comprehensive survey on continual learning in generative models.
\newblock \emph{arXiv}, 2025{\natexlab{c}}.

\bibitem[Guo et~al.(2025{\natexlab{d}})Guo, Zhang, Chen, Gao, Jiang, Wang, and Heng]{guo2025sciverse}
Ziyu Guo, Ray Zhang, Hao Chen, Jialin Gao, Dongzhi Jiang, Jiaze Wang, and Pheng-Ann Heng.
\newblock Sciverse: Unveiling the knowledge comprehension and visual reasoning of lmms on multi-modal scientific problems.
\newblock \emph{arXiv}, 2025{\natexlab{d}}.

\bibitem[Gurari et~al.(2018)Gurari, Li, Stangl, Guo, Lin, Grauman, Luo, and Bigham]{gurari2018vizwiz}
Danna Gurari, Qing Li, Abigale~J Stangl, Anhong Guo, Chi Lin, Kristen Grauman, Jiebo Luo, and Jeffrey~P Bigham.
\newblock Vizwiz grand challenge: Answering visual questions from blind people.
\newblock In \emph{Proceedings of the IEEE conference on computer vision and pattern recognition}, pages 3608--3617, 2018.

\bibitem[Han et~al.(2024)Han, Du, Du, Zhou, Wu, Zheng, and Han]{han2024slim}
Jiayi Han, Liang Du, Hongwei Du, Xiangguo Zhou, Yiwen Wu, Weibo Zheng, and Donghong Han.
\newblock Slim: Let llm learn more and forget less with soft lora and identity mixture.
\newblock \emph{arXiv}, 2024.

\bibitem[Hase et~al.(2025)Hase, Rashid, Lewis, Liu, Koike-Akino, Parsons, and Wang]{hase2025smoothed}
Ryo Hase, Md~Rafi~Ur Rashid, Ashley Lewis, Jing Liu, Toshiaki Koike-Akino, Kieran Parsons, and Ye Wang.
\newblock Smoothed embeddings for robust language models.
\newblock \emph{arXiv}, 2025.

\bibitem[He et~al.(2020)He, Zhang, Mou, Xing, and Xie]{he2020pathvqa}
Xuehai He, Yichen Zhang, Luntian Mou, Eric Xing, and Pengtao Xie.
\newblock Pathvqa: 30000+ questions for medical visual question answering.
\newblock \emph{arXiv}, 2020.

\bibitem[Herrera et~al.(2020)Herrera, Krach, and Teichmann]{herrera2020local}
Calypso Herrera, Florian Krach, and Josef Teichmann.
\newblock Local lipschitz bounds of deep neural networks.
\newblock \emph{arXiv}, 2020.

\bibitem[Hsiao et~al.(2022)Hsiao, Zubach, Baechler, Carbune, Lin, Wang, Sunkara, Zhu, and Chen]{hsiao2022screenqa}
Yu-Chung Hsiao, Fedir Zubach, Gilles Baechler, Victor Carbune, Jason Lin, Maria Wang, Srinivas Sunkara, Yun Zhu, and Jindong Chen.
\newblock Screenqa: Large-scale question-answer pairs over mobile app screenshots.
\newblock \emph{arXiv}, 2022.

\bibitem[Hu et~al.(2022)Hu, Shen, Wallis, Allen-Zhu, Li, Wang, Wang, Chen, et~al.]{hu2022lora}
Edward~J Hu, Yelong Shen, Phillip Wallis, Zeyuan Allen-Zhu, Yuanzhi Li, Shean Wang, Lu Wang, Weizhu Chen, et~al.
\newblock Lora: Low-rank adaptation of large language models.
\newblock \emph{ICLR}, 1\penalty0 (2):\penalty0 3, 2022.

\bibitem[Huai et~al.(2025)Huai, Zhou, Wu, Chen, Bai, Zhou, and He]{huai2025cl}
Tianyu Huai, Jie Zhou, Xingjiao Wu, Qin Chen, Qingchun Bai, Ze Zhou, and Liang He.
\newblock Cl-moe: Enhancing multimodal large language model with dual momentum mixture-of-experts for continual visual question answering.
\newblock In \emph{Proceedings of the Computer Vision and Pattern Recognition Conference}, pages 19608--19617, 2025.

\bibitem[Hudson and Manning(2019)]{hudson2019gqa}
Drew~A Hudson and Christopher~D Manning.
\newblock Gqa: A new dataset for real-world visual reasoning and compositional question answering.
\newblock In \emph{Proceedings of the IEEE/CVF conference on computer vision and pattern recognition}, pages 6700--6709, 2019.

\bibitem[Hussain et~al.(2023)Hussain, Liu, Sun, and Shan]{hussain2023m}
Atin~Sakkeer Hussain, Shansong Liu, Chenshuo Sun, and Ying Shan.
\newblock M2ugen: Multi-modal music understanding and generation with the power of large language models.
\newblock \emph{arXiv}, 2023.

\bibitem[Huybrechts et~al.(2025)Huybrechts, Ronanki, Jayanthi, Fitzgerald, and Veeravanallur]{huybrechts2025document}
Goeric Huybrechts, Srikanth Ronanki, Sai~Muralidhar Jayanthi, Jack Fitzgerald, and Srinivasan Veeravanallur.
\newblock Document haystack: A long context multimodal image/document understanding vision llm benchmark.
\newblock \emph{arXiv}, 2025.

\bibitem[Jaiswal et~al.(2025)Jaiswal, Liu, and Frommholz]{jaiswal2025multimodal}
Amit~Kumar Jaiswal, Haiming Liu, and Ingo Frommholz.
\newblock Multimodal rag enhanced visual description.
\newblock \emph{arXiv}, 2025.

\bibitem[Jang et~al.(2022)Jang, Ye, Lee, Yang, Shin, Han, Kim, and Seo]{jang2022temporalwiki}
Joel Jang, Seonghyeon Ye, Changho Lee, Sohee Yang, Joongbo Shin, Janghoon Han, Gyeonghun Kim, and Minjoon Seo.
\newblock Temporalwiki: A lifelong benchmark for training and evaluating ever-evolving language models.
\newblock \emph{arXiv}, 2022.

\bibitem[Johnson et~al.(2017)Johnson, Hariharan, Van Der~Maaten, Fei-Fei, Lawrence~Zitnick, and Girshick]{johnson2017clevr}
Justin Johnson, Bharath Hariharan, Laurens Van Der~Maaten, Li Fei-Fei, C Lawrence~Zitnick, and Ross Girshick.
\newblock Clevr: A diagnostic dataset for compositional language and elementary visual reasoning.
\newblock In \emph{Proceedings of the IEEE conference on computer vision and pattern recognition}, pages 2901--2910, 2017.

\bibitem[Kazemzadeh et~al.(2014)Kazemzadeh, Ordonez, Matten, and Berg]{kazemzadeh2014referitgame}
Sahar Kazemzadeh, Vicente Ordonez, Mark Matten, and Tamara Berg.
\newblock Referitgame: Referring to objects in photographs of natural scenes.
\newblock In \emph{Proceedings of the 2014 conference on empirical methods in natural language processing (EMNLP)}, pages 787--798, 2014.

\bibitem[Kembhavi et~al.(2016)Kembhavi, Salvato, Kolve, Seo, Hajishirzi, and Farhadi]{kembhavi2016diagram}
Aniruddha Kembhavi, Mike Salvato, Eric Kolve, Minjoon Seo, Hannaneh Hajishirzi, and Ali Farhadi.
\newblock A diagram is worth a dozen images.
\newblock \emph{arXiv}, 2016.

\bibitem[Kembhavi et~al.(2017)Kembhavi, Seo, Schwenk, Choi, Farhadi, and Hajishirzi]{kembhavi2017you}
Aniruddha Kembhavi, Minjoon Seo, Dustin Schwenk, Jonghyun Choi, Ali Farhadi, and Hannaneh Hajishirzi.
\newblock Are you smarter than a sixth grader? textbook question answering for multimodal machine comprehension.
\newblock In \emph{Proceedings of the IEEE Conference on Computer Vision and Pattern recognition}, pages 4999--5007, 2017.

\bibitem[Kirkpatrick et~al.(2017)Kirkpatrick, Pascanu, Rabinowitz, Veness, Desjardins, Rusu, Milan, Quan, Ramalho, Grabska-Barwinska, et~al.]{kirkpatrick2017overcoming}
James Kirkpatrick, Razvan Pascanu, Neil Rabinowitz, Joel Veness, Guillaume Desjardins, Andrei~A Rusu, Kieran Milan, John Quan, Tiago Ramalho, Agnieszka Grabska-Barwinska, et~al.
\newblock Overcoming catastrophic forgetting in neural networks.
\newblock \emph{Proceedings of the national academy of sciences}, 114\penalty0 (13):\penalty0 3521--3526, 2017.

\bibitem[Lavda et~al.(2018)Lavda, Ramapuram, Gregorova, and Kalousis]{lavda2018continual}
Frantzeska Lavda, Jason Ramapuram, Magda Gregorova, and Alexandros Kalousis.
\newblock Continual classification learning using generative models.
\newblock \emph{arXiv}, 2018.

\bibitem[Ledoux et~al.(2015)Ledoux, Nourdin, and Peccati]{ledoux2015stein}
Michel Ledoux, Ivan Nourdin, and Giovanni Peccati.
\newblock Stein’s method, logarithmic sobolev and transport inequalities.
\newblock \emph{Geometric and Functional Analysis}, 25\penalty0 (1):\penalty0 256--306, 2015.

\bibitem[Li et~al.(2024{\natexlab{a}})Li, Zhang, Zhang, Guo, Zhang, Zhang, Li, Liu, and Li]{li2024llavanext-ablations}
Bo Li, Hao Zhang, Kaichen Zhang, Dong Guo, Yuanhan Zhang, Renrui Zhang, Feng Li, Ziwei Liu, and Chunyuan Li.
\newblock Llava-next: What else influences visual instruction tuning beyond data?, 2024{\natexlab{a}}.

\bibitem[Li et~al.(2024{\natexlab{b}})Li, Zhang, Zhang, Guo, Zhang, Li, Zhang, Liu, and Li]{li2024llavanext-strong}
Bo Li, Kaichen Zhang, Hao Zhang, Dong Guo, Renrui Zhang, Feng Li, Yuanhan Zhang, Ziwei Liu, and Chunyuan Li.
\newblock Llava-next: Stronger llms supercharge multimodal capabilities in the wild, 2024{\natexlab{b}}.

\bibitem[Li et~al.(2024{\natexlab{c}})Li, Zhang, Guo, Zhang, Li, Zhang, Zhang, Li, Liu, and Li]{li2024llava1vision}
Bo Li, Yuanhan Zhang, Dong Guo, Renrui Zhang, Feng Li, Hao Zhang, Kaichen Zhang, Yanwei Li, Ziwei Liu, and Chunyuan Li.
\newblock Llava-onevision: Easy visual task transfer.
\newblock \emph{arXiv}, 2024{\natexlab{c}}.

\bibitem[Li et~al.(2024{\natexlab{d}})Li, Wong, Zhang, Usuyama, Liu, Yang, Naumann, Poon, and Gao]{li2024llava}
Chunyuan Li, Cliff Wong, Sheng Zhang, Naoto Usuyama, Haotian Liu, Jianwei Yang, Tristan Naumann, Hoifung Poon, and Jianfeng Gao.
\newblock Llava-med: Training a large language-and-vision assistant for biomedicine in one day.
\newblock \emph{Advances in Neural Information Processing Systems}, 36, 2024{\natexlab{d}}.

\bibitem[Li et~al.(2024{\natexlab{e}})Li, Zhang, Zhang, Zhang, Li, Li, Ma, and Li]{li2024llavanext-interleave}
Feng Li, Renrui Zhang, Hao Zhang, Yuanhan Zhang, Bo Li, Wei Li, Zejun Ma, and Chunyuan Li.
\newblock Llava-next: Tackling multi-image, video, and 3d in large multimodal models, 2024{\natexlab{e}}.

\bibitem[Li et~al.(2022)Li, Li, Xiong, and Hoi]{li2022blip}
Junnan Li, Dongxu Li, Caiming Xiong, and Steven Hoi.
\newblock Blip: Bootstrapping language-image pre-training for unified vision-language understanding and generation.
\newblock In \emph{International Conference on Machine Learning}, pages 12888--12900. PMLR, 2022.

\bibitem[Li et~al.(2023{\natexlab{a}})Li, Li, Savarese, and Hoi]{li2023blip}
Junnan Li, Dongxu Li, Silvio Savarese, and Steven Hoi.
\newblock Blip-2: Bootstrapping language-image pre-training with frozen image encoders and large language models.
\newblock In \emph{International conference on machine learning}, pages 19730--19742. PMLR, 2023{\natexlab{a}}.

\bibitem[Li et~al.(2023{\natexlab{b}})Li, He, Wang, Li, Wang, Luo, Wang, Wang, and Qiao]{li2023videochat}
KunChang Li, Yinan He, Yi Wang, Yizhuo Li, Wenhai Wang, Ping Luo, Yali Wang, Limin Wang, and Yu Qiao.
\newblock Videochat: Chat-centric video understanding.
\newblock \emph{arXiv}, 2023{\natexlab{b}}.

\bibitem[Li et~al.(2025)Li, Wang, and Jia]{li2025llama}
Yanwei Li, Chengyao Wang, and Jiaya Jia.
\newblock Llama-vid: An image is worth 2 tokens in large language models.
\newblock In \emph{European Conference on Computer Vision}, pages 323--340. Springer, 2025.

\bibitem[Li and Hoiem(2017)]{li2017learning}
Zhizhong Li and Derek Hoiem.
\newblock Learning without forgetting.
\newblock \emph{IEEE transactions on pattern analysis and machine intelligence}, 40\penalty0 (12):\penalty0 2935--2947, 2017.

\bibitem[Li et~al.(2024{\natexlab{f}})Li, Yang, Liu, Ma, Zhang, Yang, Sun, Liu, and Bai]{li2024monkey}
Zhang Li, Biao Yang, Qiang Liu, Zhiyin Ma, Shuo Zhang, Jingxu Yang, Yabo Sun, Yuliang Liu, and Xiang Bai.
\newblock Monkey: Image resolution and text label are important things for large multi-modal models.
\newblock In \emph{proceedings of the IEEE/CVF conference on computer vision and pattern recognition}, pages 26763--26773, 2024{\natexlab{f}}.

\bibitem[Lin et~al.(2023)Lin, Ye, Zhu, Cui, Ning, Jin, and Yuan]{lin2023video}
Bin Lin, Yang Ye, Bin Zhu, Jiaxi Cui, Munan Ning, Peng Jin, and Li Yuan.
\newblock Video-llava: Learning united visual representation by alignment before projection.
\newblock \emph{arXiv}, 2023.

\bibitem[Lin et~al.(2025)Lin, Zettlemoyer, Ghosh, Yih, Markosyan, Berges, and O{\u{g}}uz]{lin2025continual}
Jessy Lin, Luke Zettlemoyer, Gargi Ghosh, Wen-Tau Yih, Aram Markosyan, Vincent-Pierre Berges, and Barlas O{\u{g}}uz.
\newblock Continual learning via sparse memory finetuning.
\newblock \emph{arXiv}, 2025.

\bibitem[Lin et~al.(2017)Lin, Goyal, Girshick, He, and Doll{\'a}r]{lin2017focal}
Tsung-Yi Lin, Priya Goyal, Ross Girshick, Kaiming He, and Piotr Doll{\'a}r.
\newblock Focal loss for dense object detection.
\newblock In \emph{Proceedings of the IEEE international conference on computer vision}, pages 2980--2988, 2017.

\bibitem[Liu et~al.(2024{\natexlab{a}})Liu, Huang, Hou, Wang, fei Yin, Gong, Gao, and Ouyang]{Liu2024Uni3DLLMUP}
Dingning Liu, Xiaoshui Huang, Yuenan Hou, Zhihui Wang, Zhen fei Yin, Yongshun Gong, Peng Gao, and Wanli Ouyang.
\newblock Uni3d-llm: Unifying point cloud perception, generation and editing with large language models.
\newblock \emph{arXiv}, 2024{\natexlab{a}}.

\bibitem[Liu et~al.(2024{\natexlab{b}})Liu, Li, Li, and Lee]{liu2024improved}
Haotian Liu, Chunyuan Li, Yuheng Li, and Yong~Jae Lee.
\newblock Improved baselines with visual instruction tuning.
\newblock In \emph{Proceedings of the IEEE/CVF Conference on Computer Vision and Pattern Recognition}, pages 26296--26306, 2024{\natexlab{b}}.

\bibitem[Liu et~al.(2024{\natexlab{c}})Liu, Li, Li, Li, Zhang, Shen, and Lee]{liu2024llavanext}
Haotian Liu, Chunyuan Li, Yuheng Li, Bo Li, Yuanhan Zhang, Sheng Shen, and Yong~Jae Lee.
\newblock Llava-next: Improved reasoning, ocr, and world knowledge, 2024{\natexlab{c}}.

\bibitem[Liu et~al.(2024{\natexlab{d}})Liu, Li, Wu, and Lee]{liu2024visual}
Haotian Liu, Chunyuan Li, Qingyang Wu, and Yong~Jae Lee.
\newblock Visual instruction tuning.
\newblock \emph{Advances in neural information processing systems}, 36, 2024{\natexlab{d}}.

\bibitem[Liu et~al.(2024{\natexlab{e}})Liu, Li, Huang, Yang, Yu, Li, Yin, Liu, Jin, and Bai]{liu2024ocrbench}
Yuliang Liu, Zhang Li, Mingxin Huang, Biao Yang, Wenwen Yu, Chunyuan Li, Xu-Cheng Yin, Cheng-Lin Liu, Lianwen Jin, and Xiang Bai.
\newblock Ocrbench: on the hidden mystery of ocr in large multimodal models.
\newblock \emph{Science China Information Sciences}, 67\penalty0 (12):\penalty0 220102, 2024{\natexlab{e}}.

\bibitem[Lobry et~al.(2020)Lobry, Marcos, Murray, and Tuia]{lobry2020rsvqa}
Sylvain Lobry, Diego Marcos, Jesse Murray, and Devis Tuia.
\newblock Rsvqa: Visual question answering for remote sensing data.
\newblock \emph{IEEE Transactions on Geoscience and Remote Sensing}, 58\penalty0 (12):\penalty0 8555--8566, 2020.

\bibitem[Lu et~al.(2022)Lu, Mishra, Xia, Qiu, Chang, Zhu, Tafjord, Clark, and Kalyan]{lu2022learn}
Pan Lu, Swaroop Mishra, Tanglin Xia, Liang Qiu, Kai-Wei Chang, Song-Chun Zhu, Oyvind Tafjord, Peter Clark, and Ashwin Kalyan.
\newblock Learn to explain: Multimodal reasoning via thought chains for science question answering.
\newblock \emph{Advances in Neural Information Processing Systems}, 35:\penalty0 2507--2521, 2022.

\bibitem[Mallya et~al.(2018)Mallya, Davis, and Lazebnik]{mallya2018piggyback}
Arun Mallya, Dillon Davis, and Svetlana Lazebnik.
\newblock Piggyback: Adapting a single network to multiple tasks by learning to mask weights.
\newblock In \emph{Proceedings of the European conference on computer vision (ECCV)}, pages 67--82, 2018.

\bibitem[Mao et~al.(2016)Mao, Huang, Toshev, Camburu, Yuille, and Murphy]{mao2016generation}
Junhua Mao, Jonathan Huang, Alexander Toshev, Oana Camburu, Alan~L Yuille, and Kevin Murphy.
\newblock Generation and comprehension of unambiguous object descriptions.
\newblock In \emph{Proceedings of the IEEE conference on computer vision and pattern recognition}, pages 11--20, 2016.

\bibitem[Mishra et~al.(2019)Mishra, Shekhar, Singh, and Chakraborty]{mishra2019ocr}
Anand Mishra, Shashank Shekhar, Ajeet~Kumar Singh, and Anirban Chakraborty.
\newblock Ocr-vqa: Visual question answering by reading text in images.
\newblock In \emph{2019 international conference on document analysis and recognition (ICDAR)}, pages 947--952. IEEE, 2019.

\bibitem[Nguyen et~al.(2024)Nguyen, Truong, Nguyen, Dowling, Li, and Luu]{nguyen2024insect}
Hoang-Quan Nguyen, Thanh-Dat Truong, Xuan~Bac Nguyen, Ashley Dowling, Xin Li, and Khoa Luu.
\newblock Insect-foundation: A foundation model and large-scale 1m dataset for visual insect understanding.
\newblock In \emph{Proceedings of the IEEE/CVF Conference on Computer Vision and Pattern Recognition}, pages 21945--21955, 2024.

\bibitem[Nguyen et~al.(2025)Nguyen, Nguyen, Cothren, Yilmaz, and Luu]{nguyen2025hyperglm}
Trong-Thuan Nguyen, Pha Nguyen, Jackson Cothren, Alper Yilmaz, and Khoa Luu.
\newblock Hyperglm: Hypergraph for video scene graph generation and anticipation.
\newblock In \emph{Proceedings of the Computer Vision and Pattern Recognition Conference}, pages 29150--29160, 2025.

\bibitem[Nguyen et~al.(2026)Nguyen, Serna-Aguilera, Choudhary, Sinha, Li, and Luu]{nguyen2026cobra}
Xuan-Bac Nguyen, Manuel Serna-Aguilera, Arabinda~Kumar Choudhary, Pawan Sinha, Xin Li, and Khoa Luu.
\newblock Cobra: A continual learning approach to vision-brain understanding: X. nguyen et al.
\newblock \emph{International Journal of Computer Vision}, 134\penalty0 (1):\penalty0 30, 2026.

\bibitem[Qiu et~al.(2023)Qiu, Shen, Sun, Zheng, Chang, Zheng, and Wang]{sats_prj_2023}
Yiqiao Qiu, Yixing Shen, Zhuohao Sun, Yanchong Zheng, Xiaobin Chang, Weishi Zheng, and Ruixuan Wang.
\newblock Sats: Self-attention transfer for continual semantic segmentation.
\newblock \emph{Pattern Recognition}, 138:\penalty0 109383, 2023.

\bibitem[Rafailov et~al.(2023)Rafailov, Sharma, Mitchell, Manning, Ermon, and Finn]{rafailov2023direct}
Rafael Rafailov, Archit Sharma, Eric Mitchell, Christopher~D Manning, Stefano Ermon, and Chelsea Finn.
\newblock Direct preference optimization: Your language model is secretly a reward model.
\newblock \emph{Advances in neural information processing systems}, 36:\penalty0 53728--53741, 2023.

\bibitem[Razdaibiedina et~al.(2023)Razdaibiedina, Mao, Hou, Khabsa, Lewis, and Almahairi]{razdaibiedina2023progressive}
Anastasia Razdaibiedina, Yuning Mao, Rui Hou, Madian Khabsa, Mike Lewis, and Amjad Almahairi.
\newblock Progressive prompts: Continual learning for language models.
\newblock \emph{arXiv}, 2023.

\bibitem[Schulman et~al.(2017)Schulman, Wolski, Dhariwal, Radford, and Klimov]{schulman2017proximal}
John Schulman, Filip Wolski, Prafulla Dhariwal, Alec Radford, and Oleg Klimov.
\newblock Proximal policy optimization algorithms.
\newblock \emph{arXiv}, 2017.

\bibitem[Serna-Aguilera et~al.(2026)Serna-Aguilera, Anderes, Dobbs, and Luu]{serna2026nico}
Manuel Serna-Aguilera, Raegan Anderes, Page Dobbs, and Khoa Luu.
\newblock Nico-rag: Multimodal hypergraph retrieval-augmented generation for understanding the nicotine public health crisis.
\newblock \emph{arXiv}, 2026.

\bibitem[Shen et~al.(2025)Shen, Pei, Peng, Song, Liu, Peng, Sun, Hao, Wang, Zhang, and Zhou]{shen2025skyworkr1v3}
Wei Shen, Jiangbo Pei, Yi Peng, Xuchen Song, Yang Liu, Jian Peng, Haofeng Sun, Yunzhuo Hao, Peiyu Wang, Jianhao Zhang, and Yahui Zhou.
\newblock Skywork-r1v3 technical report, 2025.

\bibitem[Shi et~al.(2024{\natexlab{a}})Shi, Xu, Wang, Qin, Wang, Wang, Wang, Ebrahimi, and Wang]{shi2024continual_survey}
Haizhou Shi, Zihao Xu, Hengyi Wang, Weiyi Qin, Wenyuan Wang, Yibin Wang, Zifeng Wang, Sayna Ebrahimi, and Hao Wang.
\newblock Continual learning of large language models: A comprehensive survey.
\newblock \emph{ACM Computing Surveys}, 2024{\natexlab{a}}.

\bibitem[Shi et~al.(2024{\natexlab{b}})Shi, Hu, Bin, Liu, Yang, Ng, Bing, and Lee]{shi2024math}
Wenhao Shi, Zhiqiang Hu, Yi Bin, Junhua Liu, Yang Yang, See-Kiong Ng, Lidong Bing, and Roy Ka-Wei Lee.
\newblock Math-llava: Bootstrapping mathematical reasoning for multimodal large language models.
\newblock \emph{arXiv}, 2024{\natexlab{b}}.

\bibitem[Shi et~al.(2024{\natexlab{c}})Shi, Liu, Su, Wu, Liu, Song, and Wang]{shi2024densely}
Zenglin Shi, Pei Liu, Tong Su, Yunpeng Wu, Kuien Liu, Yu Song, and Meng Wang.
\newblock Densely distilling cumulative knowledge for continual learning.
\newblock \emph{arXiv}, 2024{\natexlab{c}}.

\bibitem[Siino(2024)]{siino2024mcrock}
Marco Siino.
\newblock Mcrock at semeval-2024 task 4: Mistral 7b for multilingual detection of persuasion techniques in memes.
\newblock In \emph{Proceedings of the 18th International Workshop on Semantic Evaluation (SemEval-2024)}, pages 53--59, 2024.

\bibitem[Sima et~al.(2024)Sima, Renz, Chitta, Chen, Zhang, Xie, Bei{\ss}wenger, Luo, Geiger, and Li]{sima2024drivelm}
Chonghao Sima, Katrin Renz, Kashyap Chitta, Li Chen, Hanxue Zhang, Chengen Xie, Jens Bei{\ss}wenger, Ping Luo, Andreas Geiger, and Hongyang Li.
\newblock Drivelm: Driving with graph visual question answering.
\newblock In \emph{European conference on computer vision}, pages 256--274. Springer, 2024.

\bibitem[Singh et~al.(2019)Singh, Natarajan, Shah, Jiang, Chen, Batra, Parikh, and Rohrbach]{singh2019towards}
Amanpreet Singh, Vivek Natarajan, Meet Shah, Yu Jiang, Xinlei Chen, Dhruv Batra, Devi Parikh, and Marcus Rohrbach.
\newblock Towards vqa models that can read.
\newblock In \emph{Proceedings of the IEEE/CVF conference on computer vision and pattern recognition}, pages 8317--8326, 2019.

\bibitem[Smith et~al.(2023)Smith, Karlinsky, Gutta, Cascante-Bonilla, Kim, Arbelle, Panda, Feris, and Kira]{smith2023coda}
James~Seale Smith, Leonid Karlinsky, Vyshnavi Gutta, Paola Cascante-Bonilla, Donghyun Kim, Assaf Arbelle, Rameswar Panda, Rogerio Feris, and Zsolt Kira.
\newblock Coda-prompt: Continual decomposed attention-based prompting for rehearsal-free continual learning.
\newblock In \emph{Proceedings of the IEEE/CVF conference on computer vision and pattern recognition}, pages 11909--11919, 2023.

\bibitem[Suhr and Artzi(2023)]{suhr2023continual}
Alane Suhr and Yoav Artzi.
\newblock Continual learning for instruction following from realtime feedback.
\newblock \emph{Advances in Neural Information Processing Systems}, 36:\penalty0 32340--32359, 2023.

\bibitem[Tanaka et~al.(2025)Tanaka, Iki, Hasegawa, Nishida, Saito, and Suzuki]{tanaka2025vdocrag}
Ryota Tanaka, Taichi Iki, Taku Hasegawa, Kyosuke Nishida, Kuniko Saito, and Jun Suzuki.
\newblock Vdocrag: Retrieval-augmented generation over visually-rich documents.
\newblock In \emph{Proceedings of the Computer Vision and Pattern Recognition Conference}, pages 24827--24837, 2025.

\bibitem[Tian et~al.(2019)Tian, Krishnan, and Isola]{tian2019contrastive}
Yonglong Tian, Dilip Krishnan, and Phillip Isola.
\newblock Contrastive representation distillation.
\newblock \emph{arXiv}, 2019.

\bibitem[Tran et~al.(2025)Tran, Truong, and Luu]{tran2025bima}
Huu-Thien Tran, Thanh-Dat Truong, and Khoa Luu.
\newblock Bima: Bijective maximum likelihood learning approach to hallucination prediction and mitigation in large vision-language models.
\newblock In \emph{Proceedings of the Computer Vision and Pattern Recognition Conference}, pages 5302--5311, 2025.

\bibitem[Truong et~al.(2021)Truong, Duong, Le, Phung, Rainwater, and Luu]{truong2021bimal}
Thanh-Dat Truong, Chi~Nhan Duong, Ngan Le, Son~Lam Phung, Chase Rainwater, and Khoa Luu.
\newblock Bimal: Bijective maximum likelihood approach to domain adaptation in semantic scene segmentation.
\newblock In \emph{Proceedings of the IEEE/CVF International Conference on Computer Vision}, pages 8548--8557, 2021.

\bibitem[Truong et~al.(2023{\natexlab{a}})Truong, Le, Raj, Cothren, and Luu]{truong2023fredom}
Thanh-Dat Truong, Ngan Le, Bhiksha Raj, Jackson Cothren, and Khoa Luu.
\newblock Fredom: Fairness domain adaptation approach to semantic scene understanding.
\newblock In \emph{IEEE/CVF Computer Vision and Pattern Recognition (CVPR)}, 2023{\natexlab{a}}.

\bibitem[Truong et~al.(2023{\natexlab{b}})Truong, Nguyen, Raj, and Luu]{truong2023fairness}
Thanh-Dat Truong, Hoang-Quan Nguyen, Bhiksha Raj, and Khoa Luu.
\newblock Fairness continual learning approach to semantic scene understanding in open-world environments.
\newblock \emph{Advances in Neural Information Processing Systems}, 36:\penalty0 65456--65467, 2023{\natexlab{b}}.

\bibitem[Truong et~al.(2024)Truong, Prabhu, Wang, Raj, Gauch, Subbiah, and Luu]{truong2024eagle}
Thanh-Dat Truong, Utsav Prabhu, Dongyi Wang, Bhiksha Raj, Susan Gauch, Jeyamkondan Subbiah, and Khoa Luu.
\newblock Eagle: Efficient adaptive geometry-based learning in cross-view understanding.
\newblock \emph{Advances in Neural Information Processing Systems}, 37:\penalty0 137309--137333, 2024.

\bibitem[Truong et~al.(2025{\natexlab{a}})Truong, Bobda, Agarwal, and Luu]{truong2025mango}
Thanh-Dat Truong, Christophe Bobda, Nitin Agarwal, and Khoa Luu.
\newblock Mango: Multimodal attention-based normalizing flow approach to fusion learning.
\newblock \emph{The Thirty-ninth Annual Conference on Neural Information Processing Systems}, 2025{\natexlab{a}}.

\bibitem[Truong et~al.(2025{\natexlab{b}})Truong, Nguyen, Nguyen, Dowling, Li, and Luu]{truong2025insect}
Thanh-Dat Truong, Hoang-Quan Nguyen, Xuan-Bac Nguyen, Ashley Dowling, Xin Li, and Khoa Luu.
\newblock Insect-foundation: A foundation model and large multimodal dataset for vision-language insect understanding.
\newblock \emph{International Journal of Computer Vision}, pages 1--26, 2025{\natexlab{b}}.

\bibitem[Truong et~al.(2025{\natexlab{c}})Truong, Prabhu, Raj, Cothren, and Luu]{truong2025falcon}
Thanh-Dat Truong, Utsav Prabhu, Bhiksha Raj, Jackson Cothren, and Khoa Luu.
\newblock Falcon: Fairness learning via contrastive attention approach to continual semantic scene understanding.
\newblock In \emph{Proceedings of the Computer Vision and Pattern Recognition Conference}, pages 15065--15075, 2025{\natexlab{c}}.

\bibitem[Truong et~al.(2025{\natexlab{d}})Truong, Tran, Son, Raj, and Luu]{truong2025directedtokens}
Thanh-Dat Truong, Huu-Thien Tran, Tran~Thai Son, Bhiksha Raj, and Khoa Luu.
\newblock Directed-tokens: A robust multi-modality alignment approach to large language-vision models.
\newblock In \emph{The Thirty-ninth Annual Conference on Neural Information Processing Systems}, 2025{\natexlab{d}}.

\bibitem[Wang et~al.(2021)Wang, Li, Zhou, Chen, Grossman, and Li]{wang2021screen2words}
Bryan Wang, Gang Li, Xin Zhou, Zhourong Chen, Tovi Grossman, and Yang Li.
\newblock Screen2words: Automatic mobile ui summarization with multimodal learning.
\newblock In \emph{The 34th Annual ACM Symposium on User Interface Software and Technology}, pages 498--510, 2021.

\bibitem[Wang et~al.(2022{\natexlab{a}})Wang, Stavrou, and Skoglund]{wang2022generalizations}
Shuchan Wang, Photios~A Stavrou, and Mikael Skoglund.
\newblock Generalizations of talagrand inequality for sinkhorn distance using entropy power inequality.
\newblock \emph{Entropy}, 24\penalty0 (2):\penalty0 306, 2022{\natexlab{a}}.

\bibitem[Wang et~al.(2025)Wang, Gao, Gu, Pu, Cui, Wei, Liu, Jing, Ye, Shao, et~al.]{wang2025internvl35}
Weiyun Wang, Zhangwei Gao, Lixin Gu, Hengjun Pu, Long Cui, Xingguang Wei, Zhaoyang Liu, Linglin Jing, Shenglong Ye, Jie Shao, et~al.
\newblock Internvl3.5: Advancing open-source multimodal models in versatility, reasoning, and efficiency, 2025.

\bibitem[Wang et~al.(2023)Wang, Chen, Ge, Xia, Bao, Zheng, Zhang, Gui, and Huang]{wang2023orthogonal}
Xiao Wang, Tianze Chen, Qiming Ge, Han Xia, Rong Bao, Rui Zheng, Qi Zhang, Tao Gui, and Xuanjing Huang.
\newblock Orthogonal subspace learning for language model continual learning.
\newblock \emph{arXiv}, 2023.

\bibitem[Wang et~al.(2024)Wang, Yu, Wang, Heng, Chen, Ye, Xie, Xie, and Zhang]{wang2024exploring}
Yidong Wang, Zhuohao Yu, Jindong Wang, Qiang Heng, Hao Chen, Wei Ye, Rui Xie, Xing Xie, and Shikun Zhang.
\newblock Exploring vision-language models for imbalanced learning.
\newblock \emph{International Journal of Computer Vision}, 132\penalty0 (1):\penalty0 224--237, 2024.

\bibitem[Wang et~al.(2022{\natexlab{b}})Wang, Zhang, Lee, Zhang, Sun, Ren, Su, Perot, Dy, and Pfister]{wang2022learning}
Zifeng Wang, Zizhao Zhang, Chen-Yu Lee, Han Zhang, Ruoxi Sun, Xiaoqi Ren, Guolong Su, Vincent Perot, Jennifer Dy, and Tomas Pfister.
\newblock Learning to prompt for continual learning.
\newblock In \emph{Proceedings of the IEEE/CVF conference on computer vision and pattern recognition}, pages 139--149, 2022{\natexlab{b}}.

\bibitem[Weng et~al.(2024)Weng, Han, He, Chang, and Zhuang]{weng2024longvlm}
Yuetian Weng, Mingfei Han, Haoyu He, Xiaojun Chang, and Bohan Zhuang.
\newblock Longvlm: Efficient long video understanding via large language models.
\newblock \emph{arXiv}, 2024.

\bibitem[Xu et~al.(2024)Xu, Wang, Wang, Chen, Pang, and Lin]{xu2024pointllm}
Runsen Xu, Xiaolong Wang, Tai Wang, Yilun Chen, Jiangmiao Pang, and Dahua Lin.
\newblock Pointllm: Empowering large language models to understand point clouds.
\newblock In \emph{ECCV}, 2024.

\bibitem[Yang et~al.(2023)Yang, Liu, Zhang, Pan, Guo, Li, Chen, Gao, Guo, and Zhang]{yang2023lidar}
Senqiao Yang, Jiaming Liu, Ray Zhang, Mingjie Pan, Zoey Guo, Xiaoqi Li, Zehui Chen, Peng Gao, Yandong Guo, and Shanghang Zhang.
\newblock Lidar-llm: Exploring the potential of large language models for 3d lidar understanding.
\newblock \emph{arXiv}, 2023.

\bibitem[Ye et~al.(2023)Ye, Yin, Zhang, Du, Chen, Wang, and Ma]{ye2023unit}
Muchao Ye, Ziyi Yin, Tianrong Zhang, Tianyu Du, Jinghui Chen, Ting Wang, and Fenglong Ma.
\newblock Unit: a unified look at certified robust training against text adversarial perturbation.
\newblock \emph{Advances in Neural Information Processing Systems}, 36:\penalty0 22351--22368, 2023.

\bibitem[Yin et~al.(2022)Yin, Li, and Xiong]{yin2022contintin}
Wenpeng Yin, Jia Li, and Caiming Xiong.
\newblock Contintin: Continual learning from task instructions.
\newblock \emph{arXiv}, 2022.

\bibitem[Ying et~al.(2024)Ying, Meng, Wang, Li, Lin, Yang, Zhang, Zhang, Lin, Liu, et~al.]{ying2024mmt}
Kaining Ying, Fanqing Meng, Jin Wang, Zhiqian Li, Han Lin, Yue Yang, Hao Zhang, Wenbo Zhang, Yuqi Lin, Shuo Liu, et~al.
\newblock Mmt-bench: A comprehensive multimodal benchmark for evaluating large vision-language models towards multitask agi.
\newblock \emph{arXiv}, 2024.

\bibitem[Yu et~al.(2024)Yu, Tang, Xu, Cui, Ran, Yan, Liu, Wang, Han, Liu, et~al.]{yu2024visrag}
Shi Yu, Chaoyue Tang, Bokai Xu, Junbo Cui, Junhao Ran, Yukun Yan, Zhenghao Liu, Shuo Wang, Xu Han, Zhiyuan Liu, et~al.
\newblock Visrag: Vision-based retrieval-augmented generation on multi-modality documents.
\newblock \emph{arXiv}, 2024.

\bibitem[Zan et~al.(2022)Zan, Chen, Yang, Lin, Kim, Guan, Wang, Chen, and Lou]{zan2022cert}
Daoguang Zan, Bei Chen, Dejian Yang, Zeqi Lin, Minsu Kim, Bei Guan, Yongji Wang, Weizhu Chen, and Jian-Guang Lou.
\newblock Cert: continual pre-training on sketches for library-oriented code generation.
\newblock \emph{arXiv}, 2022.

\bibitem[Zeng et~al.(2024)Zeng, Zhu, Guo, Zhang, and Liu]{zeng2024modalprompt}
Fanhu Zeng, Fei Zhu, Haiyang Guo, Xu-Yao Zhang, and Cheng-Lin Liu.
\newblock Modalprompt: Towards efficient multimodal continual instruction tuning with dual-modality guided prompt.
\newblock \emph{arXiv}, 2024.

\bibitem[Zhang et~al.(2024{\natexlab{a}})Zhang, Lei, Gui, Yang, He, Wang, and Xu]{zhang2024cppo}
Han Zhang, Yu Lei, Lin Gui, Min Yang, Yulan He, Hui Wang, and Ruifeng Xu.
\newblock Cppo: Continual learning for reinforcement learning with human feedback.
\newblock In \emph{The Twelfth International Conference on Learning Representations}, 2024{\natexlab{a}}.

\bibitem[Zhang et~al.(2023)Zhang, Shen, Liu, Liu, Bendersky, Najork, and Zhang]{zhang2023not}
Rongzhi Zhang, Jiaming Shen, Tianqi Liu, Jialu Liu, Michael Bendersky, Marc Najork, and Chao Zhang.
\newblock Do not blindly imitate the teacher: Using perturbed loss for knowledge distillation.
\newblock \emph{arXiv}, 2023.

\bibitem[Zhang et~al.(2024{\natexlab{b}})Zhang, Wei, Jiang, Guo, Li, Zhang, Tong, Liu, Zhou, Wei, et~al.]{zhang2024mavis}
Renrui Zhang, Xinyu Wei, Dongzhi Jiang, Ziyu Guo, Shicheng Li, Yichi Zhang, Chengzhuo Tong, Jiaming Liu, Aojun Zhou, Bin Wei, et~al.
\newblock Mavis: Mathematical visual instruction tuning with an automatic data engine.
\newblock \emph{arXiv}, 2024{\natexlab{b}}.

\bibitem[Zhang et~al.(2025{\natexlab{a}})Zhang, Gui, Sun, Feng, Xu, Zhang, Fu, Li, Hauptmann, Bisk, and Yang]{zhang-etal-2025-direct}
Ruohong Zhang, Liangke Gui, Zhiqing Sun, Yihao Feng, Keyang Xu, Yuanhan Zhang, Di Fu, Chunyuan Li, Alexander~G Hauptmann, Yonatan Bisk, and Yiming Yang.
\newblock Direct preference optimization of video large multimodal models from language model reward.
\newblock In \emph{Proceedings of the 2025 Conference of the Nations of the Americas Chapter of the Association for Computational Linguistics: Human Language Technologies (Volume 1: Long Papers)}, pages 694--717. Association for Computational Linguistics, 2025{\natexlab{a}}.

\bibitem[Zhang et~al.(2025{\natexlab{b}})Zhang, Bai, Yang, and Liang]{zhang2025c}
Xin Zhang, Liang Bai, Xian Yang, and Jiye Liang.
\newblock C-lora: Continual low-rank adaptation for pre-trained models.
\newblock \emph{arXiv}, 2025{\natexlab{b}}.

\bibitem[Zhang et~al.(2024{\natexlab{c}})Zhang, Li, Liu, Lee, Gui, Fu, Feng, Liu, and Li]{zhang2024llavanextvideo}
Yuanhan Zhang, Bo Li, haotian Liu, Yong~jae Lee, Liangke Gui, Di Fu, Jiashi Feng, Ziwei Liu, and Chunyuan Li.
\newblock Llava-next: A strong zero-shot video understanding model, 2024{\natexlab{c}}.

\bibitem[Zhao et~al.(2025)Zhao, Zhu, Guo, Wang, Wang, Meng, and Zhang]{zhao2025mllm}
Hongbo Zhao, Fei Zhu, Haiyang Guo, Meng Wang, Rundong Wang, Gaofeng Meng, and Zhaoxiang Zhang.
\newblock Mllm-cl: Continual learning for multimodal large language models.
\newblock \emph{arXiv}, 2025.

\bibitem[Zhao et~al.(2023)Zhao, Misra, Kr{\"a}henb{\"u}hl, and Girdhar]{zhao2023learning}
Yue Zhao, Ishan Misra, Philipp Kr{\"a}henb{\"u}hl, and Rohit Girdhar.
\newblock Learning video representations from large language models.
\newblock In \emph{Proceedings of the IEEE/CVF Conference on Computer Vision and Pattern Recognition}, pages 6586--6597, 2023.

\end{thebibliography}
}

\clearpage
\appendix
\maketitlesupplementary

\section{Proof of Lemmas}

\subsection{Proof of Lemma \ref{lemma:lower-bound-kl}}

The DPO loss in Eqn.~\eqref{eqn:dpo_loss} can be rewrite as follows:
\begin{equation}\label{eqn:dpo-loss-rewrite}
\footnotesize
\begin{split}
    \mathcal{L}_{\mathrm{DPO}}(\pi_t, \pi_{t-1}) &= \mathbb{E}_{x, y^+, y}\Bigg[\ell_{\beta}\left(\Delta_t(y^+, y^-) \right)\Bigg] \\
    \ell_{\beta}\left( u \right) &= \log\Big(1+\exp(\beta u)\Big) \\
    \Delta_t(y^+, y^-) &= \Big(\log \pi_t(y^+ | x)-\log \pi_t(y^- | x)\Big) \\ &\quad\quad - \Big(\log \pi_{t-1}(y^+ | x)-\log \pi_{t-1}(y^- | x)\Big)
\end{split}
\end{equation}

\begin{lemma}\label{lemma:pairwise-logistic}
\textbf{Pairwise Logistic Lower Bound by Margin}. For any $u \in \mathbb{R}$ and $\beta > 0$,
\[
\ell_\beta(u) = \log(1+e^{-\beta u})
\;\;\ge\;\; \log 2 - \tfrac{\beta}{2} u.
\]
\end{lemma}

\noindent
\textbf{Proof.} Since $\log(1+e^{-v})$ is convex and symmetric around $v=0$ in the sense that its tangent at 0 is $\log 2 - \tfrac{1}{2}v$, the global underestimator follows from convexity $\log(1+e^{-v}) \;\ge\; \log 2 - \tfrac{1}{2}v$. Then, we substitute $v=\beta u$ follow by taking expectation over pairs:
\begin{equation}\label{eq:lemma1}
\mathbb{E}[\ell_\beta(\Delta_t(y^+, y^-)) \ge \log 2 - \frac{\beta}{2}\,\mathbb{E}[\Delta_t(y^+, y^-)].
\end{equation}
As a result, the small value of the DPO loss forces large average margin $\mathbb{E}[\Delta_t(y^+, y^-)]$. In other words, the smaller value of DPO loss enforces the model’s preference for well-retained $y^+$ responses stronger than for forgotten ones $y^-.$

\begin{lemma}\label{label:ipm}
\textbf{Average Margin Controls an Integral Probability Metrics (IPM)}. Let $\mathcal{F}_1$ be the set of all 1-Lipschitz functions, 
If the reward function $r$ is $L$-Lipschitz, then $r/L \in \mathcal{F}_1$.
Then, for any pair of marginals $P^+,P^-$ where $y^+ \sim P^+(y^+)$ and $y^- \sim P^-(y^-)$, we have
\begin{equation}\label{eqn:ipm-ineq}
\footnotesize
\begin{split}
\mathbb{E}[\Delta_t(y^+, y^-)] &= \mathbb{E}_{P^+}[r(x, y^+)] - \mathbb{E}_{P^-}[r(x, y^-)] \\
    & \leq L\operatorname{IPM}_{\mathcal{F}_1}(P^+, P^-) \leq L W_1(P^+, P^-)\\
    \operatorname{IPM}_{\mathcal{F}_1}(P^+, P^-) &= \sup_{f \in \mathcal{F}_1}\left( \mathbb{E}_{y^+\sim P^+}[f(y^+)] - \mathbb{E}_{y^-\sim P^-}[f(y^-)] \right)
\end{split}
\end{equation}
where $W_1$ is the 1-Wasserstein distance.
\end{lemma}

\noindent
\textbf{Proof.} By definition of IPM over 1-Lipschitz functions and $r/L$ is admissible, the above inequality is the Kantorovich-Rubinstein duality $\mathrm{IPM}$ over 1-Lipschitz functions equals $W_1$ provided in \cite{edwards2011kantorovich}.
In addition, although Lemma \ref{label:ipm} requires $r$ to be an $L$-Lipschitz function, we have observed that a local $L$-Lipschitz reward function, which is satisfied in our setup, is also sufficient. Indeed, prior studies \cite{herrera2020local}
rigorously derives bounds on the local Lipschitz constants of deep neural networks and shows they can be meaningfully controlled despite huge global constants. 
This result indicate that the LLMs behave smoothly around their high-probability outputs. 
In our context, we only need the log-ratio to be Lipschitz on the region visited by preference pairs, not globally over all possible outputs. Transformer-based LLMs incorporate norm control, weight decay, and normalization layers, which implicitly bound gradient magnitudes and curtail abrupt jumps in logits. 
Empirically, small semantic perturbations rarely cause extreme changes in logits, suggesting local smoothness holds on the data manifold \cite{ye2023unit, hase2025smoothed}.
Thus, a locally valid Lipschitz constant suffices the requirement of Lemma \ref{label:ipm}. 
Therefore, while LLMs may not be globally Lipschitz, they plausibly satisfy the needed local Lipschitz continuity in the regions relevant to DPO, making Lemma \ref{label:ipm} still valid in practice.

In addition, it can be shown that $W_1(P^+, P^-) \le 3W_1(\pi_t, \pi_{t-1})$ follows naturally from the triangle inequality of the Wasserstein distance. 
In particular, if the preference distributions $P^+$ and $P^-$ remain close to the current and previous policies, respectively, such that 
$W_1(P^+, \pi_t) \le W_1(\pi_t, \pi_{t-1})$ and $W_1(P^-, \pi_{t-1}) \le W_1(\pi_t, \pi_{t-1})$, 
then we obtain 
\begin{equation}
\footnotesize
\begin{split}
W_1(P^+, P^-) &\le W_1(P^+, \pi_t) + W_1(\pi_t, \pi_{t-1}) + W_1(\pi_{t-1}, P^-) \\ &\le 3W_1(\pi_t, \pi_{t-1})
\end{split}
\end{equation}
This conditions are typically satisfied in the DPO training, where preference sampling is a monotone and non-expansive process, e.g., sampling candidates from a mixture (please refer to Remark 1 in Section \ref{sec:proof-upper}).
In the context of continual learning of LMMs, the inequality $W_1(P^+, P^-) \le 3W_1(\pi_t, \pi_{t-1})$ implies that the discrepancy between well-retained and forgotten knowledge is bounded by the overall policy shift between two learning steps. 
Intuitively, both $P^+$ and $P^-$ remain anchored around their respective policies, so the overall variation between them is bounded by a constant multiple of the inter-policy shift $W_1(\pi_t, \pi_{t-1})$.
Concurrently, both $P^+$ and $P^-$ remain anchored around their respective policies: $P^+$ near the current policy $\pi_t$ and $P^-$ near the previous policy $\pi_{t-1}$. Thus, if the model update between tasks is smooth, the semantic drift between memory retention and forgetting remains limited. 
This highlights that our continual DPO training enforces a stable adaptation process, where catastrophic forgetting is controlled by bounding the inter-policy Wasserstein distance.

Now, the inequality in Eqn.~\eqref{eqn:ipm-ineq} can be further rewritten as follows:
\begin{equation}\label{eqn:ipm-v2}
\small
\begin{split}
    \mathbb{E}[\Delta_t(y^+, y^-)] &= \mathbb{E}_{P^+}[r(x, y^+)] - \mathbb{E}_{P^-}[r(x, y^-)] \\ &\leq 3LW_1(\pi_t,\pi_{t-1})
\end{split}
\end{equation}

\begin{lemma}\label{leema:transport-entropy}
\textbf{A Transport–Entropy Inequality}. Since the output probability $p(y|x)$ produced by the LMM of previous learning step $\pi_{t_1}$ is computed based on the softmax on the logit scores of token $y$, we can view $\pi_{t-1}$ as a Boltzmann distribution over token sequences. Then, without a strict argument, we assume that $\pi_{t-1}$ satisfies the Talagrand $T_2(C_0)$ inequality \cite{ledoux2015stein, wang2022generalizations}:
\begin{equation}
    W_2^2(\mu,\pi_{t-1}) \le 2C_0\, D_{\mathrm{KL}}(\mu \,\|\, \pi_{t-1}) \quad \text{for all } \mu,
\end{equation}
Then, since $W_1 \leq W_2$, by substituting $\mu$ by $\pi_t$, we have the final inequality as follows:
\begin{equation}
W_1(\pi_t,\pi_{t-1}) \;\le\; \sqrt{2C_0 \, D_{\mathrm{KL}}(\pi_t \,\|\, \pi_{t-1})}
\end{equation}
\end{lemma}

\noindent
\textbf{Proof.} The proof of Talagrand $T_2(C_0)$ inequality has been shown in prior studies \cite{cattiaux2003criterion}.

\noindent
\textbf{Proof of Lemma \ref{lemma:lower-bound-kl}.}
From Lemmas \ref{lemma:pairwise-logistic}-\ref{leema:transport-entropy}, we have
\begin{equation}\label{eqn:derive}
\footnotesize
\begin{split}
\mathcal{L}_{\mathrm{DPO}}(\theta;x) & \;\ge\; \log 2 - \tfrac{\beta}{2}\,\mathbb{E}[\Delta_t] \\
& \;\ge\; \log 2 - \tfrac{3\beta}{2} L W_1(\pi_t,\pi_{t-1}) \\
& \;\ge\; \log 2 - \tfrac{3\beta}{2} L \sqrt{2C_0\, D_{\mathrm{KL}}(\pi_t\|\pi_{t-1})} \\
\Rightarrow \log 2 - \mathcal{L}_{\mathrm{DPO}}(\theta;x)
&\;\le\; \tfrac{3\beta L}{2}\sqrt{2C_0 D_{\mathrm{KL}}(\pi_t\|\pi_{t-1})}
\\ 
\Rightarrow 
D_{\mathrm{KL}}(\pi_t\|\pi_{t-1}) &\;\ge\; \frac{(\log 2 - \mathcal{L}_{\mathrm{DPO}}(\pi_t, \pi_{t-1}))^2}{\tfrac{1}{2}\beta^2 3^2 L^2 2C_0}
 \\
 D_{\mathrm{KL}}(\pi_t\|\pi_{t-1}) &\ge \frac{(\log 2 - \mathcal{L}_{\mathrm{DPO}}(\pi_t, \pi_{t-1}))^2}{\beta^2 3^2 L^2 C_0}
\end{split}
\end{equation}
Then, assume that there exists a constant $M \geq 1$ such that for all $(x, y)$,
\begin{equation}
    \frac{1}{M} \leq \frac{\pi_{t-1}(y|x)}{\pi_t(y|x)} \leq M.
\end{equation}
This ensures that the predicted distributions of the LMM model at the previous learning step $\pi_{t-1}$ and the current learning step $\pi_t$ are mutually absolutely continuous and that their density ratio is uniformly bounded. 
In other words, it prevents the predictions of the LMM from collapsing across consecutive learning steps, ensuring a stable and smooth evolution of the output distribution during the continual learning procedure.
Let $ h(x, y) = \frac{\pi_{t-1}(y|x)}{\pi_t(y|x)} $ denote the likelihood ratio. The forward and reverse KL divergence can be rewritten as follows
\begin{align}
D_{\mathrm{KL}}(\pi_{t-1}\|\pi_t)
&= \mathbb{E}_{(x,y)\sim \pi_t}\!\big[h(x,y)\log h(x,y)\big],
\\[2mm]
D_{\mathrm{KL}}(\pi_t\|\pi_{t-1})
&= \mathbb{E}_{(x,y)\sim \pi_t}\!\big[-\log h(x,y)\big],
\end{align}
where $ h(x,y) = \frac{\pi_{t-1}(y|x)}{\pi_t(y|x)} $.

\noindent
\textbf{Lower bound of $D_{\mathrm{KL}}(\pi_{t-1}\|\pi_t)$.}
The function $f(u)=u\log u$ is convex with $f''(u)=1/u$.
On the interval $[1/M, M]$, the smallest curvature is $1/M$.
By the second-order convexity bound around $u=1$,
\begin{equation}
u\log u \;\ge\; (u-1) + \frac{1}{2M}(u-1)^2.
\end{equation}
Since $\mathbb{E}_{x, y \sim \pi_t(y|x)}[h(x, y)-1]=0$, taking the expectation under $\pi_t$ will result in 
\begin{equation}\label{eqn:kl-lower}
D_{\mathrm{KL}}(\pi_{t-1}\|\pi_t)
\;\ge\;
\frac{1}{2M}\,\mathbb{E}_{\pi_t}\!\big[(h(x,y)-1)^2\big].
\tag{1}
\end{equation}

\vspace{2mm}
\noindent\textbf{Upper bound of $D_{\mathrm{KL}}(\pi_t\|\pi_{t-1})$.}
Similarly, with $g(u)=-\log u$, we have $g''(u)=1/u^2$,
and on $[1/M, M]$, the largest curvature is $M^2$.
Hence,
\begin{equation}\label{eqn:kl-upper}
-\log u \;\le\; (1-u) + \frac{M^2}{2}(u-1)^2.
\end{equation}
Then, taking expectation under $\pi_t$ will result in 
\begin{equation}
D_{\mathrm{KL}}(\pi_t\|\pi_{t-1})
\;\le\;
\frac{M^2}{2}\,\mathbb{E}_{\pi_t}\!\big[(h(x,y)-1)^2\big].
\end{equation}
From Eqn.~\eqref{eqn:kl-lower} and and Eqn.~\eqref{eqn:kl-upper}, we can obtain
\begin{align}
D_{\mathrm{KL}}(\pi_{t-1}\|\pi_t)
&\ge
\frac{1}{M^3}\,D_{\mathrm{KL}}(\pi_t\|\pi_{t-1}).
\end{align}
Then, let us define $c=M^3$. Eqn. \eqref{eqn:derive} can be further derived as follows:
\begin{equation}
\begin{split}
    D_{\mathrm{KL}}(\pi_{t-1}\|\pi_t) &\geq \frac{(\log 2 - \mathcal{L}_{\mathrm{DPO}}(\theta;x))^2}{c\beta^2 3^2 L^2 C_0} \\
    &\geq \frac{1}{C_{\mathrm{lower}}}(\log 2 - \mathcal{L}_{\mathrm{DPO}}(\theta;x))^2
\end{split}
\end{equation}
where $C_{\mathrm{lower}}=c\beta^2 3^2 L^2 C_0$.

\subsection{Proof of Lemma \ref{lemma:upper-bound-kl}}\label{sec:proof-upper}

\paragraph{Remark 1. Mixture Sampling and Monotone Labeling.}
For each prompt $x$, the output candidates are sampled from
\begin{equation}
Q_x \;=\; \alpha\,\pi_{t-1}(\cdot\mid x) \;+\; (1-\alpha)\,\pi_t(\cdot\mid x)
\quad\text{with }\alpha\in(0,1],
\end{equation}
and the selection kernel (human or reward model) chooses the preferred/dispreferred
outputs $(y^+,y^-)$ \emph{monotonically} with the underlying reward, inducing pair marginals
$P^+_x, P^-_x$ that do not expand total variation beyond what is present in $Q_x$.
Formally, we have
\begin{equation}
    \mathrm{TV}\!\big(\pi_{t-1}(\cdot\mid x),\,\pi_t(\cdot\mid x)\big)
\;\;\le\;\; \frac{1}{\alpha}\;\mathrm{TV}\!\big(P^+,P^-\big).
\end{equation}

Remark 1 is both natural and theoretically justified in the context of continual learning via DPO. 
The candidate responses of DPO are typically drawn from a mixture of the previous and current policies, $Q_x$,  to ensure balanced exposure to both past and newly adapted behaviors. 
The monotone labeling condition further indicates that the preference signal, whether derived from humans or a reward model, preserves the true reward ordering of outputs. 
Then, the total-variation inequality then follows from the data-processing principle, i.e., applying a monotone labeling kernel cannot increase statistical divergence between distributions. 
Intuitively, the preference selection process can only reveal discrepancies already present in the mixture $Q_x$, not amplify them. 
Consequently, Remark 1 enforces a bounded relationship between the divergence of the induced pairwise marginals $(P_x^+, P_x^-)$ and the divergence between the underlying policies $(\pi_{t-1}, \pi_t)$. 
In addition, this guarantees that updates to $\pi_t$ remain geometrically close to $\pi_{t-1}$, providing a stability-adaptability balance in the continual learning setting, i.e., the model can adapt to new data or tasks while preventing catastrophic forgetting.

\noindent
\textbf{Remark 2. Sign Consistency.}
The predictor $\pi_t$ is \emph{Bayes-consistent in sign} on the support of $M_x := \frac{1}{2}(P_x^+ + P_x^-)$:
\[
\mathrm{sgn}\!\big(q_\theta(z) - \tfrac12\big)
\;=\;
\mathrm{sgn}\!\big(\eta(z) - \tfrac12\big)
\quad\text{for $M_x$-a.e. }z,
\]
where $\eta(z) = \frac{dP_x^+}{d(P_x^+ + P_x^-)}(z)$ and
$q_\theta(z) = \sigma(\beta s_\theta(z))$ with $\sigma(u) = \tfrac{1}{1 + e^{-u}}$.
This remark is standard in excess-risk calibration and holds whenever the logistic excess risk is sufficiently small to ensure boundary consistency.

\begin{lemma}\label{lemma:calibration-pair}
\textbf{Logistic Calibration for Pairs}. Given $P^+$ and $P^-$, the total variation of $TV(P^+, P^-)$ will be bounded by the DPO loss:
\begin{equation}
\footnotesize
\mathrm{TV}\!\big(P^+,P^-\big)
\le
2\sqrt{2}\sqrt{\,\mathcal{L}_{\mathrm{DPO}}(\pi_t, \pi_{t-1})\;-\;\mathcal{L}_{\mathrm{DPO}}^\star(\pi_t, \pi_{t-1})\,}
\end{equation}
where $\mathcal{L}_{\mathrm{DPO}}^\star(\pi_t, \pi_{t-1})$ is the Bayes-optimal logistic pairwise loss.
\end{lemma}

\noindent
\textbf{Proof.}
Let us abbreviate $P^+=P_x^+$, $P^-=P_x^-$, and $M=\frac{1}{2}(P^+ + P^-)$.
The DPO loss and its Bayes-optimal counterpart can be written as
\begin{equation}
\begin{split}
\mathcal{L}_{\mathrm{DPO}}(\pi_t, \pi_{t-1}) &= \mathbb{E}_{Z\sim M}\!\big[\mathrm{CE}(\eta(Z), q_\theta(Z))\big] \\
\mathcal{L}_{\mathrm{DPO}}^\star(\pi_t, \pi_{t-1}) &= \mathbb{E}_{Z\sim M}\!\big[\mathrm{CE}(\eta(Z), \eta(Z))\big]
\end{split}
\end{equation}
where $\mathrm{CE}(\cdot,\cdot)$ is the binary cross-entropy function,
$\eta(\cdot)$ represents the true (Bayes-optimal) preference probability between positive and negative outcomes, $q_\theta(Z)$ is model-predicted probability obtained from the logit margin of the LMM model. 
Then, the excess DPO risk is formed as:
\begin{equation}
\begin{split}
\mathfrak{R}_{\pi_t,\pi_{t-1}}(x)
&=
\mathcal{L}_{\mathrm{DPO}}(\pi_t, \pi_{t-1})
-
\mathcal{L}_{\mathrm{DPO}}^\star(\pi_t, \pi_{t-1})
\\ &=
\mathbb{E}_{Z\sim M}
\!\Big[
\mathrm{KL}\big(\mathrm{Bern}(\eta(Z))\,\big\|\,\mathrm{Bern}(q_\theta(Z))\big)
\Big].
\end{split}
\end{equation}
where $\mathrm{Bern}$ is the Bernoulli distribution.

\noindent
\textbf{Bernoulli Pinsker Inequality.}
For any $z$, Pinsker’s inequality for Bernoulli distributions gives
\begin{equation}
\mathrm{KL}\big(\mathrm{Bern}(\eta(z))\,\|\,\mathrm{Bern}(q_\theta(z))\big)
\;\ge\;
2\,(\eta(z)-q_\theta(z))^2.
\end{equation}
Hence,
\begin{equation}
\label{eq:pinsker}
\mathbb{E}_M[(\eta-q_\theta)^2]
\;\le\;
\frac{1}{2}\,
\mathfrak{R}_{\pi_t,\pi_{t-1}}(x).
\end{equation}

\noindent
\textbf{Sign Consistency Implies Margin Control.}
Under Remark 2, since $\eta$ and $q_\theta$ lie on the same side of $\tfrac12$ for almost every $z$:
\begin{equation}\label{eq:cons-sign-imc}
\footnotesize
\begin{split}
|\eta(z)-\tfrac12|
\le
|\eta(z)-q_\theta(z)| + |q_\theta(z)-\tfrac12|
& \le
2\,|\eta(z)-q_\theta(z)| \\
\Rightarrow |2\eta(z)-1| &\le 4\,|\eta(z)-q_\theta(z)|
\end{split}
\end{equation}
By definition of total variation, we have
\begin{equation}
\mathrm{TV}(P^+,P^-)
=
\mathbb{E}_{Z\sim M}\!\big[|2\eta(Z)-1|\big].
\end{equation}
Then, applying Eqn. \eqref{eq:cons-sign-imc} and Cauchy–Schwarz, we will receive
\begin{equation}
\mathrm{TV}(P^+,P^-)
\;\le\;
4\,\mathbb{E}_M|\eta-q_\theta|
\;\le\;
4\,\sqrt{\mathbb{E}_M(\eta-q_\theta)^2}.
\end{equation}
Then, substitute Eqn.\eqref{eq:pinsker} will result in
\begin{equation}
\begin{split}
\mathrm{TV}(P^+,P^-)
&\le
4\,\sqrt{\tfrac{1}{2}\,\mathfrak{R}_{\pi_t,\pi_{t-1}}(x)} \\
&=
2\sqrt{2}\,\sqrt{\,
\mathcal{L}_{\mathrm{DPO}}(\pi_t, \pi_{t-1})
-
\mathcal{L}_{\mathrm{DPO}}^\star(\pi_t, \pi_{t-1})
\,}.
\end{split}
\end{equation}

\noindent
\textbf{Proof of Lemma \ref{lemma:upper-bound-kl}}. 
By Pinsker’s inequality, we have
\begin{equation}
\begin{split}
\mathrm{TV}\!\big(\pi_{t-1},\pi_t\big)
&\;\le\;
\sqrt{\tfrac{1}{2}\,D_{\mathrm{KL}}\!\big(\pi_{t-1}\,\|\,\pi_t\big)},
\\
\Rightarrow D_{\mathrm{KL}}\!\big(\pi_{t-1}\,\|\,\pi_t\big)
&\;\le\;
2\,\mathrm{TV}\!\big(\pi_{t-1},\pi_t\big)^2.
\end{split}
\end{equation}
Then, substituting Remark 1 and Lemma \ref{lemma:calibration-pair} into Pinsker’s relation will result in
\begin{equation}
\footnotesize
\begin{split}
D_{\mathrm{KL}}\!\big(\pi_{t-1}\,\|\,\pi_t\big)
&\le
2
\left(
\frac{2\sqrt{2}}{\alpha}
\sqrt{\,
\mathcal{L}_{\mathrm{DPO}}(\pi_t, \pi_{t-1})
-
\mathcal{L}_{\mathrm{DPO}}^\star(\pi_t, \pi_{t-1})
\,}
\right)^{2} \\
&\le \frac{16}{\alpha^2}\left(\mathcal{L}_{\mathrm{DPO}}(\pi_t, \pi_{t-1})- \mathcal{L}_{\mathrm{DPO}}^\star(\pi_t, \pi_{t-1})\right)\\
&\le \frac{16}{\alpha^2}\mathcal{L}_{\mathrm{DPO}}(\pi_t, \pi_{t-1}) \\
&= C_{\mathrm{upper}}\mathcal{L}_{\mathrm{DPO}}(\pi_t, \pi_{t-1}) 
\end{split}
\end{equation}
where $C_{\mathrm{upper}}=\frac{16}{\alpha^2}$.

\subsection{Proof of Lemma \ref{lema:fair-dpo-loss}}

\textbf{Proof.} Since $\log p \le 0$, one has $0 \;\le\; \alpha_\gamma(p) \;\le\; (1-p)^{\gamma}$, and $(1-p)^{\gamma}\to 0$ exponentially as $\gamma\to\infty$. Then, for any fixed $p\in(0,1)$, we have $\lim_{\gamma\to\infty}\alpha_\gamma(p) = 0$. As a result, for each group $k$, if $p(z)\in(0,1)$ a.s. and $\mathbb E[\|(p_\theta-1)\nabla s_\theta\|\mid G_k]<\infty$, then $\lim_{\gamma\to\infty} w_k^\gamma(\theta)=0$. By definition, we have
\begin{equation}
\|B_\gamma(\theta)\|
\;\le\; \sum_{k=1}^K |q_k-q'_k|\, |w_k^\gamma(\theta)|\, \|m_k(\theta)\|.
\end{equation}
Since $w_k^\gamma(\theta)\to 0$ for each $k$ as $\gamma\to\infty$, the sum tends to $0$.

\section{DPO Data}

\subsection{Data Description}

We release the dataset, annotations, and data preparation scripts on the project website \url{https://uark-cviu.github.io/projects/Fai-DPO/}.
The repository provides the finalized annotations, data structure, and detailed instructions for reproducing the dataset used in our experiments. Due to copyright restrictions, we do not directly redistribute the raw images, as the original image copyrights remain with their respective data providers. Instead, we provide guidelines that allow users to obtain the images from the official sources and reconstruct the complete dataset in a reproducible manner.

This approach ensures that the dataset can be fully reproduced while respecting the licensing terms of the original image collections. The released resources include standardized annotations, metadata, and preprocessing protocols necessary to replicate the experimental setup reported in this paper.

\subsection{Copyright and Usage Notice}

All images used in this work remain the intellectual property of their original creators and dataset providers. We do not claim ownership of any raw images. The project website provides only annotations, data splits, and utilities for downloading the images from their respective official sources.

Users are responsible for ensuring that their use of the images complies with the licensing agreements and usage terms specified by the original datasets. By following the provided instructions, researchers can reconstruct the full dataset while maintaining compliance with copyright regulations and promoting transparent and reproducible research.

\section{Data Quality and Noise Robustness}
\subsection{Data Quality}
Similar to prior work that leverages LLMs for large-scale instruction or preference data generation \cite{liu2024improved, liu2024visual, zhang-etal-2025-direct}, LLMs are only used to generate candidate negative preference pairs, which are then manually verified and filtered before training.
This verification process was conducted by five specialists over approximately one month.

\subsection{Noise Robustness}

Large-scale preference annotations in complex settings inevitably introduce noise and potential bias. However, this noise in DPO-style training is typically considered imperfect or ambiguous negative preference pairs, especially when certain negatives are overrepresented. From this perspective, our $\phi$-DPO is designed to limit the influence of any single subset of preference annotations during optimization (Lemma \ref{lema:fair-dpo-loss}), so that ambiguous or over-represented negative pairs do not negatively affect gradient updates. 
To further examine this issue, we conduct additional experiments by injecting controlled noise into the preference data. As shown in Table \ref{tab:rebut-data}, $\phi$-DPO maintains stable performance as noise increases, providing empirical evidence that the training process remains effective and balanced even in real-world settings with imperfect preference annotations.

\begin{table}[!ht]
\centering
\vspace{-4mm}
\caption{Noise Robustness Examination on the Preference Data}
\vspace{-2mm}
\label{tab:rebut-data}
\resizebox{1.0\linewidth}{!}{
\begin{tabular}{@{}c|ccccc|cccc@{}}
\hline
             Noise Level & RS & Med & AD & Sci & Fin & MFT$\uparrow$  & MFN$\uparrow$  & MAA$\uparrow$  & BWT$\uparrow$   \\
\hline

0\% &  \textbf{85.68} & \textbf{69.74} & \textbf{57.73} & \textbf{61.55} & \textbf{95.28} & \textbf{74.29} & \textbf{74.00} & \textbf{75.68} & \textbf{-0.37} \\
5\% &  84.15 & 68.73 &  57.26 & 60.94 & 95.27 & 74.09 & 73.27 &  75.18  &  -1.02 \\
15\% &  83.06 & 67.51 & 56.16 & 60.40 & 94.47 & 73.39 & 	 72.32 & 	 74.42 & 	 -1.33 \\
\hline
\end{tabular}
}
\vspace{-4mm}
\end{table}

\section{Efficiency Analysis and Scalability}

$\phi$-DPO is designed to be parameter-efficient and memory-aware by training via LoRA, without updating full model parameters. As in \cref{tab:rebut-lora}, for a 7B model, $\phi$-DPO requires only $\sim$0.65GB additional memory for weights of LoRA adapters and incurs a modest runtime overhead.
Moreover, $\phi$-DPO does not rely on replay buffers or task-specific gradient storage; preference updates are computed on-the-fly, keeping memory usage largely independent of the number of tasks. We also demonstrated scalability with different model sizes (\cref{tab:ablation-llm}), i.e., InternVL-7B and LLaVA-13B, showing consistent performance with manageable computation.

\begin{table}[!ht]
\centering
\footnotesize
\vspace{-2mm}
\caption{Efficiency Analysis of Parameter-efficient Training}
\vspace{-2mm}
\label{tab:rebut-lora}
\setlength{\tabcolsep}{10pt}
\resizebox{1.0\linewidth}{!}{
\begin{tabular}{@{}l|ccc@{}}
\hline
Method         & LoRA+FT  \cite{hu2022lora} & MoELoRA \cite{chen2024coin} & $\phi$-DPO \\
               \hline
Training Time  & 8.11s/iter & 8.35s/iter & 9.77s/iter   \\
Adapter Memory & $\sim$0.65G     & $\sim$0.65G & $\sim$0.65G       \\
\hline
\end{tabular}
}
\vspace{-4mm}
\end{table}

\section{Ablation on Weighted Loss Coefficients}

We adopt a fixed weight of losses in Eqn. \eqref{eqn:final-cl-mllm} to analyze the stability–plasticity trade-off. The SFT favors adaptation to new tasks, while $\phi$-DPO stabilizes updates and mitigates forgetting. As in \cref{rebut:weight_loss_coeff}, higher SFT weight improves adaptation (higher MFT) but increases forgetting (lower BWT), whereas higher $\phi$-DPO weight improves retention at the cost of adaptability.

\begin{table}[H]
\caption{Ablation on $\lambda_\textrm{SFT}$ and $\lambda_{\phi\textrm{-DPO}}$}
\label{rebut:weight_loss_coeff}
\vspace{-2mm}
\resizebox{1.0\linewidth}{!}{
\begin{tabular}{@{}cc|ccccc|cccc@{}}
\hline
             $\lambda_\textrm{SFT}$ & $\lambda_{\phi\textrm{-DPO}}$ & RS & Med & AD & Sci & Fin & MFT$\uparrow$  & MFN$\uparrow$  & MAA$\uparrow$  & BWT$\uparrow$   \\
\hline
2.0 & 1.0 &  83.62 & 68.56 & 57.18 & 61.32 & \textbf{96.43} & \textbf{75.38} 	 & 73.42 	 & \textbf{76.09} 	 & -2.45 \\
1.0 & 1.0 &  \textbf{85.68} & \textbf{69.74} & \textbf{57.73} & \textbf{61.55} & {95.28} & {74.29} & \textbf{74.00} & {75.68} & {-0.37} \\
1.0 & 2.0 &  83.71 & 67.54 & 55.67 & 59.23 & 93.16 & 72.13 &  71.86 & 73.61 &  \textbf{-0.33} \\
\hline
\end{tabular}
}
\vspace{-4mm}
\end{table}

\end{document}